\documentclass[preprint,11pt]{article}

\usepackage{comment,url,algorithm,algorithmic,graphicx,subcaption}
\usepackage{amssymb,amsfonts,amsmath,amsthm,amscd,dsfont,mathrsfs,mathtools,nicefrac}
\usepackage{float,psfrag,epsfig,color,url,hyperref}
\usepackage{epstopdf,bbm,mathtools,enumitem}

\makeatletter
\@addtoreset{equation}{section}
\makeatother

\newtheorem{theorem}{Theorem}[section]

\newtheorem{lemma}[theorem]{Lemma}

\newtheorem{proposition}[theorem]{Proposition}

\newtheorem{corollary}[theorem]{Corollary}
\newtheorem{definition}{Definition}
\newtheorem{remark}{Remark}

\def\se{{\psi_1}}
\def\sg{{\psi_2}}

\def\S{\mathcal{S}}
\def\N{\mathcal{N}}

\def\d{\mathrm{d}}
\def\B{\mathcal{B}}
\def\YY{\mathcal{Y}}
\def\XX{\mathcal{X}}
\def\D{\mathbf{D}}
\def\Bols{\B^\delta_\Sig(\betaOLS)}
\def\Bolst{\B_\delta(\tilde{\beta}^{\sf ols})}
\def\Ev{\mathcal{E}}
\def\lmin{\lambda_{\text{min}}}
\def\lmax{\lambda_{\text{max}}}

\def\dcgf{{\cgf^{(1)}}}
\def\ddcgf{{\cgf^{(2)}}}
\def\dddcgf{{\cgf^{(3)}}}
\def\X{\bold{X}}

\def\I{\bold{I}}
\def\e{{\epsilon}}

\def\Sig{{\boldsymbol \Sigma}}
\def\SigT{\widetilde{\Sig}}
\def\SigH{\widehat{\boldsymbol {\Sigma}}}
\def\soft{{\sf soft}}
\def\supp{\mathrm{supp}}

\def\betaGLM{\beta^{\sf pop}}
\def\betaOLS{\beta^{\sf ols}}
\def\betaOLStilde{\tilde{\beta}^{\sf ols}}

\def\betaGLMreg{\beta^{\sf pop}}

\def\betaOLShat{\hat{\beta}^{\sf ols}}
\def\betaOURhat{\hat{\beta}^{\sf\hspace{.01in} sls}}
\def\betaRIDGE{\beta^{\sf ols}}
\def\betaLASSO{{\beta}^{\sf lasso}}

\def\tbeta{\tilde{\beta}}
\def\tmu{\tilde{\mu}}

\def\hz{\hat{\zeta}}
\def\cgf{\Psi} 
\def\c{c}
\def\hc{\hat{\c}}
\def\hy{\hat{y}}
\def\tx{w}

\def\by{\gamma}

\def\<{\langle}
\def\>{\rangle}
\def\normal{{\sf N}}
\def\reals{{\mathbb R}}
\def\sumton{\sum_{i=1}^n}
\def\argmin{\operatorname{argmin} \displaylimits}

\def\minimize{\text{minimize}}
\newcommand{\eq}[1]{\begin{align}#1\end{align}}
\newcommand{\eqn}[1]{\begin{align*}#1\end{align*}}
\renewcommand{\P}[1]{\mathbb{P}\left( #1 \right)}
\newcommand{\E}[1]{\mathbb{E}\left[#1\right]}

\newcommand{\var}[1]{\mathrm{Var}\left(#1\right)}

\newcommand{\expp}[1]{\exp\left\{#1\right\}}
\newcommand{\logg}[1]{\log\left(#1\right)}
\newcommand{\maxx}[1]{\max\left\{#1\right\}}
\newcommand{\minn}[1]{\min\left\{#1\right\}}

\newcommand{\abs}[1]{\left|#1\right|}
\newcommand{\norm}[1]{\left\|#1\right\|_2}
\newcommand{\normO}[1]{\left\|#1\right\|_1}
\newcommand{\normI}[1]{\left\|#1\right\|_\infty}

\newcommand{\normSG}[1]{\left\|#1\right\|_\sg}
\newcommand{\normSE}[1]{\left\|#1\right\|_\se}
\newcommand{\inner}[1]{{\langle #1 \rangle}}

\renewcommand{\O}[1]{\mathcal{O}\left(#1\right)}

\begin{document}

\title{{Scalable Approximations for Generalized Linear Problems}}

\author{
  Murat A.~Erdogdu\thanks{
  Department of Statistics,
  Stanford University, \texttt{erdogdu@stanford.edu}}
 \; \; \; \; 
   Mohsen Bayati\thanks{
   Graduate School of Business,
   Stanford University,
  \texttt{bayati@stanford.edu}}
   \; \; \; \;
   Lee H.~Dicker\thanks{
   Department of Statistics and Biostatistics,
   Rutgers University and Amazon 
   (Work conducted while at Rutgers University),
   \texttt{ldicker@stat.rutgers.edu},
   }
}

\date{November 16, 2016}
\maketitle

\begin{abstract}
In stochastic optimization,
the population risk is generally approximated by the empirical risk.
However, in the large-scale setting, minimization of the empirical risk may be 
computationally restrictive.
In this paper, we design
an efficient algorithm to approximate the population risk minimizer
in generalized linear problems such as
binary classification with surrogate losses and generalized linear regression models.
We focus on large-scale problems,
where the iterative minimization of the empirical risk is computationally intractable, i.e.,
the number of observations $n$
is much larger than the dimension of the parameter $p$, i.e. $n \gg p \gg 1$.
We show that under random sub-Gaussian design,
the true minimizer of the population risk is approximately proportional to the
corresponding ordinary least squares (OLS) estimator.
Using this relation,
we design an algorithm that achieves the same accuracy as 
the empirical risk minimizer
through iterations that 
attain up to a cubic convergence rate,
and that are cheaper than 
any batch optimization algorithm by at least a factor of $\mathcal{O}(p)$.
We provide theoretical guarantees for our algorithm,
and analyze the convergence behavior in terms of data dimensions.
Finally, we demonstrate the performance of 
our algorithm on well-known classification and regression problems,
through extensive numerical studies 
on large-scale datasets, and show that
it achieves the highest performance compared to several other widely
used and specialized optimization algorithms.
\end{abstract}
\section{Introduction}
\label{sec::intro}
We consider the following optimization problem
\eq{\label{eq::population-risk}
  \underset{\beta \in \reals^p}{\text{minimize}} \ R(\beta) \coloneqq 
  \E{\vphantom{\big[}\cgf\left(\inner{x, \beta}\right)- y \inner{x, \beta}},
}
where $\cgf : \reals \to \reals$ is a non-linear function,
$y \in \YY \subset \reals$ denotes the response variable, 
$x \in\XX \subset \reals^p$ denotes the predictor (or covariate),
and the expectation is over the joint distribution of $(y, x)$.
The above minimization is called a generalized linear problem in its
canonical representation, and it is commonly encountered in the statistical learning. 
Celebrated examples include binary classification 
with smooth surrogate losses \cite{buja2005loss, reid2010composite}, 
and generalized linear models (GLMs) such as 
Poisson regression, logistic regression, ordinary least squares,
multinomial regression and many applications involving graphical models
\cite{nelder1972generalized,mccullagh1989generalized,
wainwright2008graphical,koller2009probabilistic}.
These methods play a crucial role in 
numerous machine learning and statistics problems,
and 
provide a miscellaneous framework for many regression and classification tasks.  

The exact minimization of the stochastic optimization problem \eqref{eq::population-risk},
requires the knowledge of the underlying distribution of the variables $(y,x)$. 
In practice, however, the joint distribution is not available.
Therefore, after observing $n$ independent data points $(y_i,x_i)$, 
the standard approach is
to minimize the empirical risk approximation given as
\eq{\label{eq::empirical-risk}
  \underset{\beta \in \reals^p}{\text{minimize}} \ \widehat{R}(\beta) \coloneqq 
 \frac{1}{n}\sum_{i=1}^n  \cgf\left(\inner{x_i, \beta}\right)- y_i \inner{x_i, \beta}\, .
}
In the case of GLMs, the empirical risk minimization given 
in \eqref{eq::empirical-risk} is called the maximum likelihood
estimation, whereas
in the case of binary classification, it is generally referred to as 
surrogate loss minimization. Due to non-linear structure
of the optimization task given in \eqref{eq::empirical-risk},
for both problems, the minimization 
of the empirical risk requires iterative methods.
Regardless of the problem formulation, 
the most commonly used optimization method is the Newton-Raphson method, 
which may be viewed as a reweighted least squares algorithm
\cite{mccullagh1989generalized, buja2005loss}.
This method uses a second-order approximation to benefit
from the curvature of the log-likelihood and achieves
locally quadratic convergence. A drawback of this
approach is its excessive per-iteration cost of $\mathcal{O}(np^2)$.
To remedy this, Hessian-free Krylov sub-space based methods
such as conjugate gradient and minimal residual are used, 
but the resulting direction is imprecise 
\cite{hestenes1952methods,paige1975solution,martens2010deep}.
%
On the other hand, first-order approximation yields the gradient descent
algorithm, which attains a linear convergence rate 
with $\mathcal{O}(np)$ per-iteration cost. 
Although its convergence rate is slow compared to 
that of the second-order methods, 
its modest per-iteration cost makes it practical for large-scale problems. 
In the regime $n \gg p$, another popular optimization technique
is the class of Quasi-Newton methods 
\cite{bishop1995neural,nesterov2004introductory},
which can attain a per-iteration cost of $\mathcal{O}(np)$, 
and the convergence rate is locally super-linear; 
a well-known member of this class of methods is the 
BFGS algorithm
\cite{broyden1970convergence,fletcher1970new,goldfarb1970family,shanno1970conditioning}.
There are recent studies 
that exploit the special structure of GLMs
\cite{erdogdu2015newton-stein},
and achieve near-quadratic convergence with
a per-iteration cost of $\O{np}$,
and an additional cost of covariance estimation.

In this paper, we take an alternative approach for minimizing \eqref{eq::population-risk},
based on an identity that is well-known in some areas of statistics, but
appears to have received relatively little attention for its computational
implications in large-scale problems. Let $\betaGLM$ denote the true minimizer
of the population risk given in \eqref{eq::population-risk}, 
and let $\betaOLS$ denote the corresponding
ordinary least squares (OLS) coefficients defined as 
$\betaOLS = \E{xx^T}^{-1}\E{xy}$.  Then, under certain random
predictor (design) models, 
\begin{equation} \label{eq::proportionality}
  \betaGLM\propto\betaOLS.  
\end{equation}
For logistic regression with Gaussian design (which is equivalent to
Fisher's discriminant analysis), 
\eqref{eq::proportionality} was noted by Fisher in the 1930s
\cite{fisher1936use}; a more general formulation for models with
Gaussian design is given in
\cite{brillinger2012generalized}.  The relationship
\eqref{eq::proportionality} suggests that if the constant of
proportionality is known, then $\betaGLM$ can be estimated
by computing the OLS estimator, which may be substantially simpler than
minimizing the empirical risk.
In fact, in some applications
like binary classification,
it may not be necessary to find the constant of proportionality in
\eqref{eq::proportionality}.
Our work in this paper builds on this idea.

Our contributions can be summarized as follows.
\begin{enumerate}
\item 
  We show that  $\betaGLM$
  is approximately proportional to $\betaOLS$ in the random design setting,
  regardless of the covariate (predictor) distribution.
  That is, we prove
  \[
    \normI{\betaGLM - c_{\cgf} \times\betaOLS} \lesssim \frac{1}{p},
  \]
  for some $\c_\cgf \in \reals$ which depends on the non-linearity $\cgf$.
  Our generalization uses zero-bias transformations \cite{goldstein1997stein}.
We also show that the above relation still holds under certain types of regularization.
\item We design a computationally efficient estimator
  for $\betaGLM$ by first estimating the OLS coefficients,
  and then estimating the proportionality constant $c_{\Psi}$ via line search.
  We refer to the resulting estimator as 
  the Scaled Least Squares (SLS) estimator 
  and denote it by $\betaOURhat$.  After estimating the OLS
  coefficients, 
  the second step of our algorithm involves finding a root of a real
  valued function; this can be accomplished using iterative
  methods with up to a cubic convergence rate
  and only $\mathcal{O}(n)$ per-iteration cost. 
  This is cheaper than
  the classical batch methods mentioned above 
  by at least a factor of $\mathcal{O}(p)$.

\item For random design with sub-Gaussian predictors, we show that
  \[
    \normI{\betaOURhat - \betaGLM} 
    \lesssim  \frac{1}{p} + 
    \sqrt{\frac{p}{n/\log(n)}}.
  \]
  This bound characterizes the performance of the 
  proposed estimator in terms of data dimensions, 
  and justifies the use of the algorithm in the regime $n \gg p \gg 1$.

\item We demonstrate how to transform 
a binary classification problem with smooth surrogate loss
into a generalized linear problem, and how our methods can
be applied to obtain a computationally efficient optimization scheme.
We further discuss the canonicalization of the square loss, which may be of 
independent interest to non-convex optimization community.

\item We propose a scalable algorithm for converting one generalized linear
problem to another by exploiting the proportionality relation \eqref{eq::proportionality}. 
The proposed algorithm requires only $\O{n}$ per each iteration,
with no additional cost.

\item We study the statistical and computational performance of $\betaOURhat$, 
  and compare it to that
  of the empirical risk minimizer (using several well-known implementations), 
  on a variety of large-scale datasets. 
\end{enumerate}

The rest of the paper is organized as follows:
Section \ref{sec::related-work} surveys the related work
and Section \ref{sec::preliminaries} introduces the
required background and the notation.
In Section \ref{sec::equivalence}, we provide 
the intuition behind the relationship \eqref{eq::proportionality}, which
are based on exact calculations for the Gaussian design setting.
In Section \ref{sec::algorithm}, 
we propose our algorithm and discuss its computational properties. 
Theoretical results are given in Section \ref{sec::theory}.
In Section \ref{sec::conversion}, we propose an algorithm to
convert one GLM type to another.
We discuss how a binary classification problem can be cast as a
generalized linear problem in Section \ref{sec::binary-classification},
and in Section \ref{sec::canonicalization}
we propose a method to canonicalize the square loss.
Section \ref{sec::experiments}
provides a thorough comparison between the
proposed algorithm and other existing methods.
Finally, we conclude with a brief discussion in Section \ref{sec::discussion}.

\subsection{Related work} \label{sec::related-work}

As mentioned in Section \ref{sec::intro}, the relationship
\eqref{eq::proportionality} is well-known in several forms in
statistics.  Brillinger \cite{brillinger2012generalized} derived
\eqref{eq::proportionality} for models with Gaussian predictors using Stein's lemma. 
Li \& Duan \cite{li1989regression} studied model misspecification problems
in statistics
and derived
\eqref{eq::proportionality} when the predictor distribution has linear
conditional means (this is a slight generalization of Gaussian
predictors). 
The relation \eqref{eq::proportionality} has led to various techniques for
dimension reduction \cite{li1991sliced,li2009dimension}, 
and more recently, it has been studied by \cite{plan2015generalized,thrampoulidis2015lasso}
in the context of compressed sensing.
It has been shown that the standard lasso estimator may be very effective when
used in models where the relationship between the expected response and
the signal is nonlinear, and the predictors (i.e. the design or
sensing matrix) are Gaussian.  A common theme for all of this previous
work is that it focuses solely on settings where
\eqref{eq::proportionality} holds exactly and the predictors are
Gaussian (or, in the case of \cite{li1989regression}, very nearly
Gaussian).   Two key novelties of the present paper are (i) our focus on
the 
computational benefits following from \eqref{eq::proportionality} for large scale
problems  with $n \gg p \gg 1$; and (ii) our rigorous finite sample analysis of models with
non-Gaussian predictors, where \eqref{eq::proportionality} is shown to
be approximately valid.  To the best of our knowledge, the present paper 
and its earlier version \cite{erdogdu2016scaled} are the first to
consider the relation \eqref{eq::proportionality} in the context of optimization.

\section{Preliminaries and notation}
\label{sec::preliminaries}
We assume a random design setting,  
where the observed data consists of $n$ random iid pairs $(y_1,x_1)$, $(y_2,x_2)$, $\ldots$,
$(y_n,x_n)$; $y_i \in \YY \subset \reals$ is the response variable and 
$x_i = (x_{i1},\ldots,x_{ip})^T \in \XX \subset \reals^{p}$ is 
the vector of predictors or covariates.
We focus on problems where the minimization \eqref{eq::population-risk} is desirable,
but we do not need to assume that $(y_i,x_i)$ are actually drawn from
a particular distribution or the corresponding statistical model 
(i.e. we allow for model misspecification).

\begin{equation}\label{eq::glm}
  \betaGLM = \argmin_{\beta \in \reals^p}\
  \E{\vphantom{\big[}\Psi(\<x_i,\beta\>) - y_i\<x_i,\beta\>}.
\end{equation}
While we make no assumptions on $\Psi$ beyond
smoothness, note that when the optimization problem is GLM, 
and $\Psi$ is the cumulant generating function
for $y_i\mid x_i$, then the problem reduces to the standard GLM with canonical
link and regression parameters $\betaGLM$
\cite{mccullagh1989generalized}.
Examples of GLMs in
this form include logistic regression with $\Psi(w) = \log\{1+e^w\}$,
Poisson regression with $\Psi(w) = e^w$, and linear regression (least squares)
with $\Psi(w) = w^2/2$.

Our objective is to find a computationally efficient estimator for
$\betaGLM$.   
The alternative estimator for $\betaGLM$ proposed in this paper is
related to the OLS coefficient vector, which is defined by 
$\betaOLS \coloneqq \mathbb{E}[{x_ix_i^T}]^{-1}\E{x_iy_i}$; the
corresponding OLS estimator is $\betaOLShat \coloneqq
(\X^T\X)^{-1}\X^Ty$, where $\X = (x_1,\ldots,x_n)^T$ is the $n \times
p$ design matrix and $y =
(y_1,\ldots,y_n)^T \in \reals^n$.

Additionally, throughout the text we let $[m]\!=\! \{1,2,...,m \}$,
for positive integers $m$, and 
we denote the size of a set $S$ by $|S|$. The $m$-th derivative of
a function $g:\reals \to \reals$ is denoted by $g^{(m)}$.
For a vector $u \in \reals^p$ and a $n \times p$ matrix $\mathbf{U}$, 
we let $\|u\|_q$ and $\|\mathbf{U}\|_q$ denote the
$\ell_q$-vector and -operator norms, respectively.  If $S \subseteq
[n]$, let $\mathbf{U}_S$ denote the $|S|\times p$ matrix obtained from
$\mathbf{U}$ by
extracting the rows that are indexed by $S$. 
For a symmetric matrix $\mathbf{M} \in \reals^{p\times p}$, 
$\lmax(\mathbf{M})$ and $\lmin(\mathbf{M})$ 
denote the maximum and minimum eigenvalues, respectively, and 
$\rho_k(\mathbf{M})$ denotes the condition number of $\mathbf{M}$ with respect to $k$-norm.
We denote by $\normal_q$ the $q$-variate normal distribution,
and all expectations are over all randomness inside the brackets.
Finally, we use $a \lesssim b$ and $a \leq \O{b}$ interchangeably,
whichever is convenient (where $\O{\cdot}$ refers to the big O notation).

\section{OLS is equivalent to the true minimizer up to a scalar factor}  
\label{sec::equivalence}


To motivate our methodology, we assume in this section that 
the covariates are multivariate normal, as in \cite{brillinger2012generalized}.
These distributional assumptions will be relaxed in
Section \ref{sec::theory}.

\begin{proposition}\label{prop::equivalence}
  Assume that the covariates are multivariate normal 
  with mean 0 and covariance matrix $\Sig$, i.e.
  $x_i \sim \normal_p(0,\Sig)$. Then $\betaGLM$ 
  can be written as
  \eq{\label{eq::prop-lemma}
    \betaGLM =  \c_{\cgf} \times \betaOLS,
  }
  where $\c_\cgf \in \reals$ is the fixed point of the mapping
  \eq{
  z \to \E{\ddcgf(\<x_i,\betaOLS\>z)}^{-1}.
  }
\end{proposition}

\begin{proof}[Proof of Proposition \ref{prop::equivalence}]
  %
  The optimal point in the optimization problem \eqref{eq::glm},
  has to satisfy the following normal equations,
  \eq{\label{eq::normal-eq}
    \E{yx_i} = \E{x_i \dcgf(\<x_i,\beta\>)}.
  }
  Now, denote by $\phi(x\mid\Sig)$ the multivariate normal density with 
  mean 0 and covariance matrix $\Sig$. 
  We recall the well-known property of Gaussian density 
  $\d\phi(x\mid\Sig)/\d x = -\Sig^{-1} x \phi(x\mid\Sig)$.
  Using this and
  integration by parts on the right hand side of the above equation,
  we obtain
  \eq{\label{eq::stein}
    \E{x_i \dcgf(\<x_i,\beta\>)} =& \int x \dcgf(\<x,\beta\>) \phi(x\mid\Sig) \ \d
    x, \\
    =& \Sig \beta \ \underbrace{\mathbb{E}\big[\ddcgf(\<x_i,\beta\>)\big]}_{\in\  \reals},\nonumber
}
which is basically the Stein's lemma. Combining this with the normal equations
\eqref{eq::normal-eq} and multiplying both side with $\Sig^{-1}$, we obtain the desired result. 

\end{proof}

Proposition \ref{prop::equivalence} and its proof provide the main intuition behind
our proposed method. 
Observe that in our derivation,
we only worked with the right hand side of the normal equations
\eqref{eq::normal-eq} which does not depend on the
response variable $y_i$. 
Therefore, the equivalence will hold
regardless of the joint distribution of $(y_i,x_i)$.
This is the main difference from the proof of \cite{brillinger2012generalized}
where $y_i$ is assumed to follow a single index model.
In Section \ref{sec::theory}, where we extend the method to
non-Gaussian predictors, the identity \eqref{eq::stein} is generalized via the
zero-bias transformations \cite{goldstein1997stein}.


\subsection{Regularization}
\label{sec::regularization}
A version of Proposition \ref{prop::equivalence} incorporating regularization
--- an important tool for datasets where $p$ is large
relative to $n$ or the predictors are highly collinear --- is also
possible, as outlined briefly in this section.  We focus on $\ell^2$-regularization (ridge regression) in this
section; some connections with lasso ($\ell^1$-regularization) are
discussed in Section \ref{sec::theory} and Corollary \ref{cor::lasso}.

For $\lambda \geq 0$, define the $\ell_2$-regularized empirical risk minimizer,
\eq{
  \betaGLMreg_\lambda = \argmin_{\beta\in \reals^p}\ 
  \E {\cgf(\inner{x_i, \beta}) -y_i\inner{x_i,\beta} } +
  \frac{\lambda}{2} \norm{\beta}^2
}
and the corresponding $\ell^2$-regularized OLS coefficients $
\betaRIDGE_\lambda = \left(\E{x_ix_i^T} + \lambda \I \right)^{-1}\E{x_iy_i}$
(so $\betaGLM = \betaGLMreg_0$ and $\betaOLS = \betaRIDGE_0$). The
same argument as above implies that
\eq{
  \betaGLMreg_\lambda =
  c_\cgf \times \betaRIDGE_{\gamma}, \ \text{ where }\ \gamma = \lambda c_\cgf.
}
This suggests that the ordinary ridge regression for
the linear model can be used to estimate the $\ell^2$-regularized
empirical risk minimizer
$\betaGLMreg_{\lambda}$.  Further pursuing these ideas
for problems where regularization is a critical issue 
may be an interesting area for future research.  
%

\section{SLS: Scaled Least Squares estimator}
\label{sec::algorithm}

\begin{algorithm}[t]
  \caption{SLS: Scaled Least Squares Estimator}
  \label{alg::1}
  \begin{algorithmic}
    \STATE {\bfseries Input:} Data $(y_i,x_i)_{i=1}^n$
    \STATE {\bfseries Step 1. Compute the least squares estimator: 
      $\betaOLShat$ and $\hy = \X\betaOLShat$.}
    \STATE \hspace{.39in} For a sub-sampling based OLS
    estimator, let $S \subset [n]$ be a
    \STATE \hspace{.39in} random subset and take  $\betaOLShat = \frac{|S|}{n}
    (\X^T_S\X_S^\text{\vphantom{T}} )^{-1}\X^Ty$.
    \STATE {\bfseries Step 2. Solve the following equation for $\c \in \reals$: 
      $1 = \frac{c}{n} \sum_{i=1}^n \ddcgf(\c\, \hy_i)$.}
    \STATE \hspace{.39in} Use Newton's root-finding method:
    \STATE \hspace{.65in} Initialize $c$;
    \STATE \hspace{.65in} Repeat until convergence:
    \vspace{.05in}
    \STATE \hspace{.75in} 
    $\c \leftarrow \c -\cfrac{\c\, \frac{1}{n} \sum_{i=1}^n \ddcgf(\c\, \hy_i)-1}
    {\frac{1}{n} \sum_{i=1}^n \left\{\ddcgf(\c\, \hy_i) + \c\,\hy_i\dddcgf(\c\, \hy_i)\right\} } $. 
    \STATE \hspace{.55in} {\bfseries }
    \vspace{.01in}
    \STATE {\bfseries Output}: $\betaOURhat = c \times\betaOLShat$.
  \end{algorithmic}
\end{algorithm} 

Motivated by the results in the previous section, 
we design a computationally efficient algorithm that approximates 
the stochastic optimization problem \eqref{eq::population-risk}
that is as simple as solving the least squares problem; it is
described in Algorithm \ref{alg::1}.
The algorithm has two basic steps. 
First, we estimate the OLS
coefficients, and then in the second step
we estimate the proportionality constant via
a simple root-finding algorithm.

There are numerous fast optimization methods to solve the
least squares problem, and
even a superficial review of these could go beyond the page limits of this paper.
We emphasize that this step (finding the OLS estimator) does not have to be iterative 
and it is the main computational cost
of the proposed algorithm. 
We suggest using a sub-sampling based estimator for $\betaOLS$,
where we only use a subset of the observations to 
estimate the covariance matrix. Let $S \subset [n]$ be a random
sub-sample and denote by $\X_S$ the sub-matrix formed by the
rows of $\X$ in $S$. Then the sub-sampled OLS estimator is given
as $\betaOLShat = \big(\frac{1}{|S|}\X^T_S\X_S^\text{\vphantom{T}} \big)^{-1} \frac{1}{n}\X^Ty$.
Properties of 
sub-sampling and sketching based estimators have been well-studied 
\cite{vershynin2010introduction,dhillon2013subsampling, erdogdu2015convergence, pilanci2015newton, roosta2016sub}.
For sub-Gaussian covariates, it suffices to use
a sub-sample size of $\O{p \log(p)}$ \cite{vershynin2010introduction}.
Hence, this step requires a single time computational cost of 
$\O{|S|p^2+p^3+np}\approx \O{p\max\{p^2 \log(p),n\}}$.
For other approaches, we refer reader to 
\cite{rokhlin2008fast,Drineas:2011, dhillon2013subsampling, erdogdu2015convergence} 
and the references therein.

\begin{figure}[t]
  \centering
  \includegraphics[width=.45\linewidth]{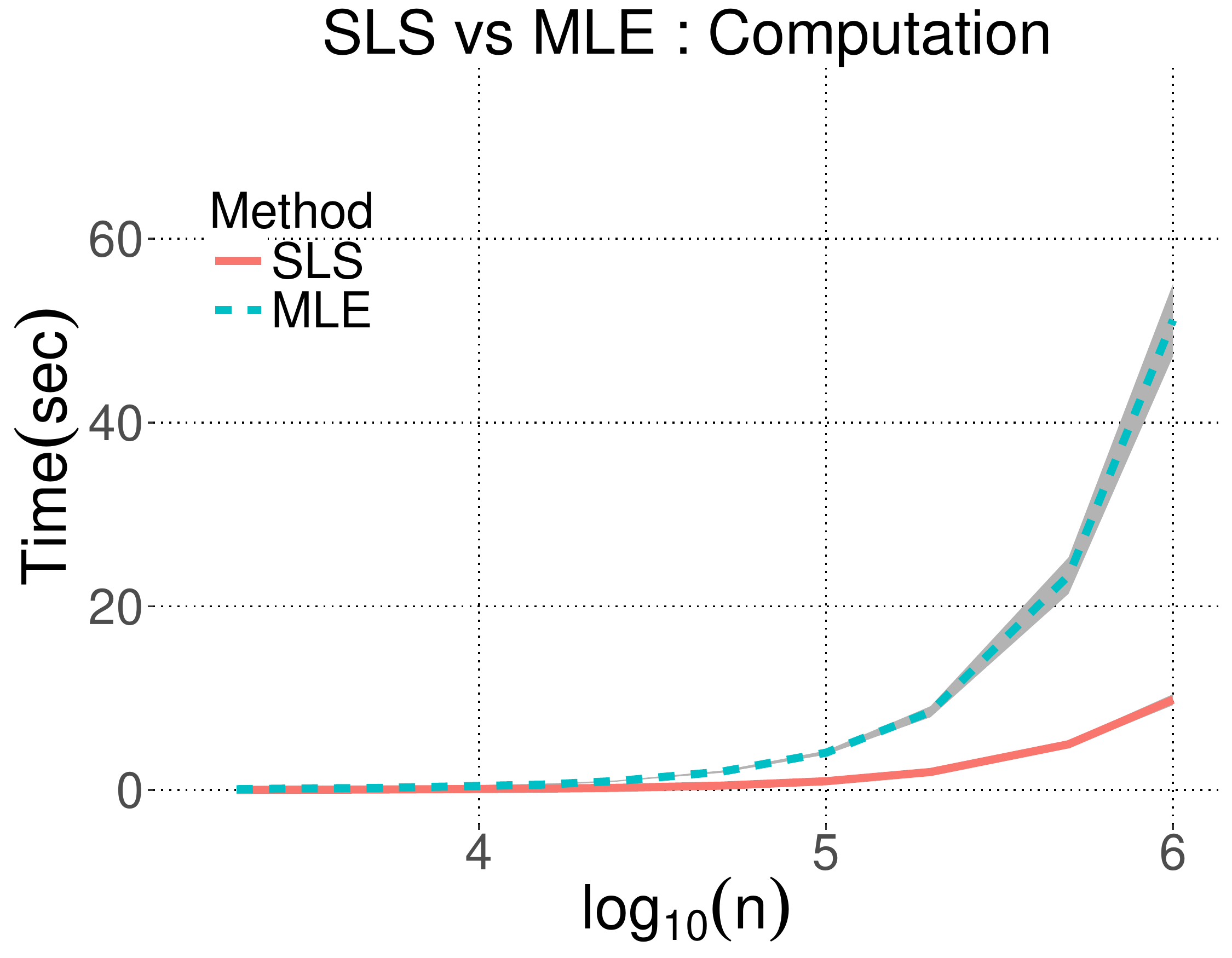}
  \includegraphics[width=.45\linewidth]{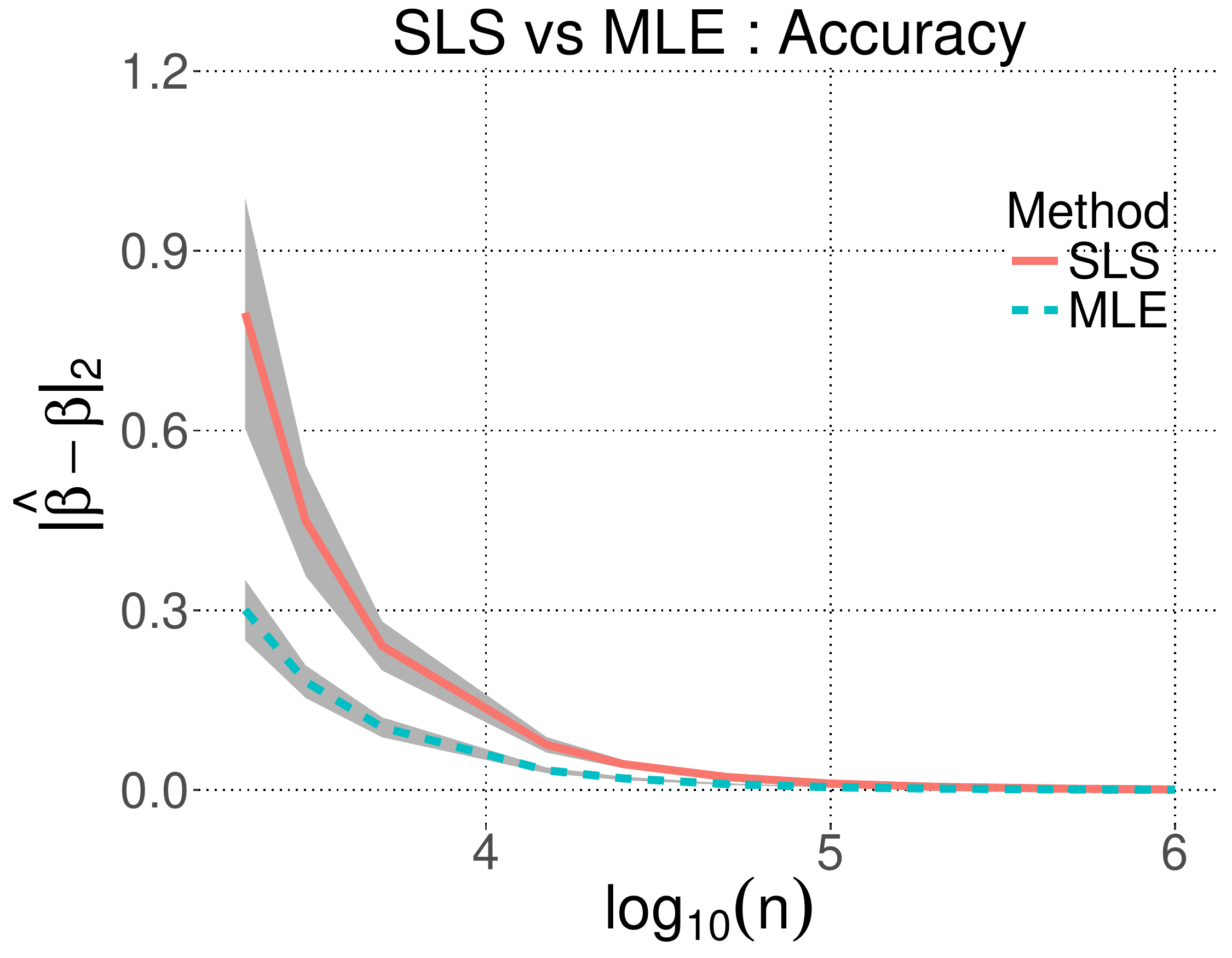}
  \caption{\small
    \label{fig::mle-vs-our}
    Logistic regression with iid standard Gaussian design.  The left plot shows the computational cost (time) for finding
    the MLE and SLS as $n$ grows and $p=200$.
    The right plot depicts the accuracy of the estimators.
    In the regime where the MLE is expensive to 
    compute, the SLS is found much more rapidly and has the same accuracy.
    \texttt{R}'s built-in functions are used to find the MLE.
  }
\end{figure}
The second step of Algorithm \ref{alg::1} involves solving a simple root-finding
problem. As with the first step of the algorithm, there are numerous
methods available for completing this task.
Newton's root-finding method with quadratic convergence or 
Halley's method with cubic convergence may be appropriate choices.
We highlight that this step costs only $\O{n}$ per-iteration and that we 
can attain up to a cubic rate of convergence.
The resulting per-iteration cost is cheaper than 
other commonly used batch algorithms by at least
a factor of $\O{p}$ --- indeed, the cost of computing the gradient is $\O{np}$.
For simplicity, we use Newton's root-finding method.

Correct initialization of the scaling constant $c$ depends on the optimization problem.
For example, in the case of GLM problems,
assuming that the GLM is a good approximation to the true conditional distribution,
by the law of total variance and basic properties of GLMs, we have
\eq{
  \var{y_i} = \E{\var{y_i \mid x_i}} 
  + \var{\E{y_i\mid x_i}}\approx c_\cgf^{-1}
  + \text{Var}\big(\dcgf(\inner{x_i,\beta})\big).
} 
It follows that 
the initialization $c = 2/\var{y_i}$ is reasonable as long as $c_\cgf^{-1} \approx
\E{\var{y_i \mid x_i}} $ is not much smaller than 
$\text{Var}\big(\dcgf(\inner{x_i,\beta})\big)$.
Our experiments show that SLS is very robust to initialization.

In Figure \ref{fig::mle-vs-our}, we compare the performance of our
SLS estimator to that of the MLE in a GLM optimization problem, 
when both are used to analyze synthetic data
generated from a logistic regression model 
under general Gaussian design with randomly generated covariance matrix.
The left plot shows the computational cost of obtaining both
estimators as $n$ increases for fixed $p$. 
The right plot shows the accuracy of the estimators.
In the regime $n \gg p \gg 1$
---  where the MLE is hard to compute ---
the MLE and the SLS achieve the same accuracy,
yet SLS has significantly smaller computation time.
We refer the reader to Section \ref{sec::theory}
for theoretical results characterizing the finite sample behavior of
the SLS.


\section{Theoretical results}
\label{sec::theory}

In this section, we use the zero-bias transformations 
\cite{goldstein1997stein}
to generalize the equivalence relation given in the previous section to the
settings where the covariates are non-Gaussian.
\begin{definition}\label{def::zb}
  Let $z$ be a random variable with mean 0 and variance $\sigma^2$. 
  Then, there exists a random variable $z^*$ that satisfies
  $
  \E{zf(z)}= \sigma^2\mathbb{E}[f^{(1)}(z^*)],
  $
  \ for all differentiable functions $f$.  The distribution of $z^*$ is
  said to be the $z$-zero-bias distribution.  
\end{definition}
The existence of $z^*$ in Definition \ref{def::zb} is a consequence of 
Riesz representation theorem \cite{goldstein1997stein}.  The normal
distribution is the unique distribution whose zero-bias transformation
is itself (i.e. the normal distribution is a fixed point of the
operation mapping the distribution of $z$ to that of $z^*$ -- which is basically Stein's lemma).

To provide some intuition behind the usefulness of the zero-bias
transformation, we refer back to the proof of Proposition
\ref{prop::equivalence}.  For simplicity, assume that the covariate vector $x_i$ has 
iid entries with mean 0, and variance 1.
Then the zero-bias transformation applied to
the $j$-th normal equation in \eqref{eq::normal-eq} yields
\eq{\label{eq::zer-bias-normal}
  \underbrace{
    \E{y_ix_{ij}} = \E{x_{ij}\dcgf\big(x_{ij}\beta_j +\Sigma_{k \neq j}x_{ik} \beta_k\big)}
  }_\text{$j$-th normal equation}
  =
  \underbrace{
    \beta_j\E{\ddcgf\left(x_{ij}^*\beta_j +\Sigma_{k \neq j}x_{ik} \beta_{ik} \right)}
  }_\text{Zero-bias transformation}.
}
The distribution of $x_{ij}^*$ is the $x_{ij}$-zero-bias distribution
and is entirely
determined by the distribution of $x_{ij}$; general properties of
$x_{ij}^*$ can be found, for example, in \cite{chen2010normal}.
If $\beta$ is well spread, it turns out that taken together, with $j =
1,\ldots,p$, the far right-hand side
in \eqref{eq::zer-bias-normal} behaves similar to the right side of
\eqref{eq::stein}, with $\Sig = \I$; that is, the behavior is similar
to the Gaussian case, where the proportionality relationship given in
Proposition \ref{prop::equivalence} holds.  This argument leads to an
approximate proportionality relationship for problems with
non-Gaussian predictors, which, when carried out rigorously, yields
the following result.

\begin{theorem}{}
  \label{thm::population-bound}
  Suppose that the whitened covariates $w_i = \Sig^{-\nicefrac{1}{2}}x_i$ are 
  independent with mean 0, covariance $\I$, and have sub-Gaussian norm bounded by $\kappa$.
  Furthermore, $w_i$'s have
  constant first and second conditional moments, i.e., 
  $\forall j \in [p]$ and $\tbeta =  \Sig^{\nicefrac{1}{2}}\betaGLM$,
  $\mathbb{E}[{w_{ij}\big|\Sigma_{k\neq j} \tbeta_k w_{ik}}]$ and 
  $\mathbb{E}[{w_{ij}^2\big|\Sigma_{k\neq j} \tbeta_k w_{ik}}]$ are constant.
  Let $\|\betaGLM\|_2 = \tau$ and assume $\betaGLM$ is $r$-well-spread 
  in the sense that $\tau / \normI{\betaGLM} = r \sqrt{p}$
  for some $r \in (0,1]$, and the function $\ddcgf$
  is Lipschitz continuous with constant $k$.
  Then, for $c_\cgf = 1/\E{\ddcgf(\inner{x_i,\betaGLM})}$, and 
  $\rho = \rho_\infty(\Sig^{1/2})$ denoting the condition number of $\Sig^{1/2}$, we have
  \eq{
    \normI{ \frac{1}{\c_\cgf}\times\betaGLM -   \betaOLS} 
    \leq  \frac{\eta}{p},\ 
    \text{ where } \ 
    \eta = 8 k \kappa^3\rho\|\Sig^{1/2}\|_\infty(\tau/r)^2.
  }
\end{theorem}

Theorem \ref{thm::population-bound} is proved in the Appendix.  
It implies that the population parameters
$\betaOLS$ and $\betaGLM$ are approximately equivalent up to a scaling factor,
with an error bound of $\O{1/p}$.
The assumption that $\betaGLM$ is well-spread can be
relaxed with minor modifications.  
For example, if we have a sparse coefficient vector, 
where $\supp(\betaGLM) = \{j; \
\betaGLM_j \neq 0\}$ is the support set of $\betaGLM$, 
then Theorem \ref{thm::population-bound} holds with $p$ 
replaced by the size of the support set.

The assumptions on the conditional moments are the relaxed versions of assumptions
that are commonly encountered in dimension reduction techniques. 
For example, sliced inverse regression methods
assume that the first conditional moment $\E{x\big|\inner{x,\beta}}$ is linear in $x$ for all $\beta$ 
\cite{li1989regression,li1991sliced},
which is satisfied by elliptically distributed random vectors.
 An important case that is not covered by these methods is the independent coordinate case, i.e.,  
 when the whitened covariates have 
independent, but not necessarily identical entries.
It is straightforward to observe that this case satisfies the assumptions of Theorem \ref{thm::population-bound}.
We refer reader to \cite{li2009dimension}, for a good
review of dimension reduction techniques and their corresponding assumptions.
We also highlight that our moment assumptions can be relaxed further,
at the expense of introducing some additional complexity into the results.

An interesting consequence of Theorem \ref{thm::population-bound} and
the remarks following the theorem is that
whenever an entry of $\betaGLM$ is zero, the corresponding entry
of $\betaOLS$ has to be small, and conversely.
For $\lambda \geq 0$, define the lasso coefficients
\eq{
  \betaLASSO_{\lambda} = \argmin_{\beta \in \reals^p}\ 
  \frac{1}{2}\E{(y_i-\inner{x_i,\beta})^2} + \lambda \normO{\beta}.
}
\begin{corollary} \label{cor::lasso}
  For any $\lambda \geq \eta/ \vert \supp(\betaGLM) \vert$,
  if $\E{x_i} = 0$ and $\E{x_ix_i^T} = \I$, 
  we have
  $
  \supp(\betaLASSO) \subset \supp(\betaGLM).
  $
  Further, if $\lambda$ and $\betaGLM$ also satisfy that $\forall j \in \supp(\betaGLM)$,
  $|\betaGLM_j| > \c_\cgf\left( \lambda + \eta/\vert \supp(\betaGLM) \vert\right)$,
  then we have
  $
  \supp(\betaLASSO) = \supp(\betaGLM).
  $
\end{corollary}

So far in this section, we have only discussed properties of the population
parameters, such as $\betaGLM$ and $\betaOLS$. In the remainder of this section, we
turn our attention to results for the estimators that are the main
focus of this paper; these results ultimately build on our earlier
results, i.e. Theorem
\ref{thm::population-bound}.

In order to precisely describe the performance of $\betaOURhat$, we
first need bounds on the OLS estimator.  The OLS estimator has been
studied extensively in the literature; however, for our purposes, we
find it convenient to derive a new bound on its accuracy.  While we
have not seen this exact bound elsewhere, 
it is very similar to Theorem 5 of \cite{dhillon2013subsampling}.

\begin{proposition}{}
  \label{prop::ols-rate}
  Assume that $\E{x_i} = 0$, $\E{x_ix_i^T} = \Sig$, and that
  $\Sig^{-\nicefrac{1}{2}}x_i$ and $y_i$ are 
  sub-Gaussian with norms $\kappa$ and $\gamma$, respectively.
  For $\lmin$ denoting the smallest eigenvalue of $\Sig$,
  and $|S| > \eta p $,
  \eq{
    \norm{\betaOLShat - \betaOLS} \leq
    \eta\lmin^{\scalebox{.8}{ $-\nicefrac{1}{2}$}}\sqrt{\frac{p}{|S|}},
  }
  with probability at least $1-3e^{-p}$, 
  where $\eta$ depends only on $\gamma$ and $\kappa$.  
\end{proposition}

Proposition \ref{prop::ols-rate} is proved in the Supplementary Material.  Our main
result on the performance of $\betaOURhat$ is given next.  
\begin{theorem}
  \label{thm::main-bound}
  Let the assumptions of Theorem \ref{thm::population-bound} 
  and Proposition \ref{prop::ols-rate} hold with 
  $\mathbb{E}[\|\Sig^{-\nicefrac{1}{2}}x\|_2] = \tmu \sqrt{p}$. 
  Further assume that the function $f(z)= z\E{\ddcgf(\inner{x,\betaOLS}z)}$
  satisfies $f(\bar{c})>1+\bar{\delta}\sqrt{p}$ for some $\bar{c}$ 
  and $\bar{\delta}$ such that
  the derivative of $f$ in the interval $[0,\bar{c}]$
  does not change sign, i.e., its absolute value is lower bounded by $\upsilon>0$.
  Then, for $n$ and $|S|$ sufficiently large, 
  with probability at least $1-5e^{-p}$, we have
  \eq{\label{eq::main-bound-1}
    \normI{\betaOURhat - \betaGLM} 
    \leq \eta_1 \frac{1}{p} + 
    \eta_2 \sqrt{\frac{p}{\minn{n/\log(n),|S|}}},
  }
  where the constants $\eta_1$ and $\eta_2$ are defined by
  \eq{\label{eq::main-bound-2}
    \eta_1 =& \eta k \bar{c}\kappa^3\rho\|\Sig^{1/2}\|_\infty(\tau/r)^2\\
    \eta_2 \label{eq::main-bound-3}
    =&  \eta\bar{c}\lmin^{-1/2} 
    \left(1+ \upsilon^{-1}\lmin^{1/2} \|\betaOLS\|_\infty\maxx{(b+k/\tmu),
        k \bar{c}\kappa} \right),
  }
  and $\eta > 0$ is a constant depending on $\kappa$ and $\gamma$.
\end{theorem}

Note that the convergence rate of the upper bound in 
\eqref{eq::main-bound-1} depends on the sum of 
the two terms, both of which 
are functions of the data dimensions $n$ and $p$.
The first term on the right in \eqref{eq::main-bound-1} comes from 
Theorem \ref{thm::population-bound},
which bounds the discrepancy between
$c_{\Psi} \times \betaOLS$ and $\betaGLM$.
This term is small
when $p$ is large, and it does not depend on the 
number of observations $n$.

The second term in the upper bound \eqref{eq::main-bound-1} comes
from estimating $\betaOLS$
and $c_\cgf$. This term is increasing in $p$, which reflects the fact
that estimating $\betaGLM$ is more challenging when $p$ is large. 
As expected, this term is decreasing in $n$ and $|S|$, 
i.e. larger sample size
yields better estimates.  
When the full OLS solution is used ($|S|=n$),
the second term becomes 
$\mathcal{O}(\sqrt{p\log(n)/n} )$,
which 
suggests that $n/ \log(n)$ should be at least of order $p$
for good performance.
Also, note that there is a theoretical threshold for the sub-sampling size $|S|$,
namely $\O{n/\log(n)}$, beyond which further sub-sampling provides no improvement.
This suggests that the sub-sampling size should be smaller than $\O{n/\log(n)}$.


\section{Converting One GLM to Another by Scaling}
\label{sec::conversion}
In this section, we describe an efficient algorithm to transform a generalized linear model 
to another.
It is often the case that a practitioner would like to change the loss function (equivalently the model)
he/she uses based on its performance. When the dataset is large, 
training a new model from the scratch is computationally inefficient and will be time consuming.
In the following, we will use the proportionality relation to transition between different loss functions.

Assume that a practitioner fitted a GLM using the loss function (or cumulant generating function) $\cgf_1$, 
but he/she would like to train a new model using the loss function $\cgf_2$.
Instead of maximizing the log-likelihood based on $\cgf_2$,
one can exploit the proportionality relation and obtain the coefficients
for the new GLM problem.
\begin{algorithm}[t]
  \caption{Conversion from one GLM to another}
  \label{alg::2}
  \begin{algorithmic}
    \STATE {\bfseries Input:} Data $(y_i,x_i)_{i=1}^n$,  and $\hat{\beta}^{\text{glm}}_1$
    \STATE {\bfseries Step 1. Compute $\hy = \X\hat{\beta}^{\normalfont\text{glm}}_1$, 
      and $\kappa = \frac{1}{n}\sumton \cgf_1^{(2)}(\hy_i)$.}
    \vspace{.1in}

    \STATE {\bfseries Step 2. Solve the following equation for $\rho \in \reals$:
      $\kappa
      = \frac{\rho}{n} \sum_{i=1}^n \cgf_2^{(2)}(\hy_i\rho)$}
    \STATE \hspace{.39in} Use Newton's root-finding method:
    \STATE \hspace{.65in} Initialize $\rho=1$;
    \STATE \hspace{.65in} Repeat until convergence:
    \vspace{.05in}
    \STATE \hspace{.75in} 
    $\rho \leftarrow \rho -
    \cfrac{\rho\, \frac{1}{n} \sum_{i=1}^n \cgf_2^{(2)}(\rho\, \hy_i)-\kappa}
    {\frac{1}{n} \sum_{i=1}^n \left\{\cgf_2^{(2)}(\rho\, \hy_i) 
        + \rho\, \hy_i \cgf_2^{(3)}(\rho\, \hy_i)\right\} } $. 
    \STATE \hspace{.55in} {\bfseries }
    \vspace{.01in}
    \STATE {\bfseries Output}: $\hat{\beta}^{\text{glm}}_2 = 
    \rho \times\hat{\beta}^{\text{glm}}_1$.
  \end{algorithmic}
\end{algorithm} 
Denote by $\betaGLM_1$ and $\betaGLM_2$ the GLM coefficients corresponding
to the loss functions $\cgf_1$ and $\cgf_2$, respectively.
We have
\eqn{
  \frac{1}{\c_{\cgf_1}}\betaGLM_1 = \frac{1}{\c_{\cgf_2}}\betaGLM_2 = \betaOLS,
}
that is, both coefficients are proportional to the OLS coefficients which does not depend on the loss function.
Therefore, these coefficients $\betaGLM_1$ and $\betaGLM_2$ are also proportional to each other and we can write
\eq{\label{eq::prop2glms}
  \betaGLM_2 = \frac{\c_{\cgf_2}}{\c_{\cgf_1}}\ \betaGLM_1 \coloneqq \rho\ \betaGLM_1,
}
where the proportionality constant between two GLM types turns out to be
the ratio between $c_{\cgf_1}$ and $c_{\cgf_2}$, i.e. $\rho = \c_{\cgf_2} / \c_{\cgf_1}$.
Using the definition of $\c_{\cgf_2}$, we write
\eqn{
  1 &= \c_{\cgf_2}\ \E{\cgf_2^{(2)}(\inner{x, \betaGLM_2})},\\
  & = \c_{\cgf_1}\rho\ \E{\cgf_2^{(2)}(\inner{x, \betaGLM_1} \rho)}.
}
Dividing the both sides by $\c_{\cgf_1}$ and using the equality $1/\c_{\cgf_1} = \E{\cgf_1^{(2)}(\inner{x, \betaGLM_1})}$, we obtain
\eqn{
  \E{\cgf_1^{(2)}(\inner{x, \betaGLM_1})} = 
  \rho\ \E{\cgf_2^{(2)}(\inner{x, \betaGLM_1} \rho)}.
}
The above equation only involves $\betaGLM_1$ as the coefficients
(which is already assumed to be known or fitted by the practitioner).
Therefore, if we solve it for the ratio $\rho$, 
we can estimate $\betaGLM_2$ by simply
using the proportionality relation given in \eqref{eq::prop2glms}. 

The procedure described above is summarized as Algorithm \ref{alg::2}. 
We emphasize that
this procedure does not require the computation of the OLS estimator which was the
main cost of SLS. The procedure only requires a per-iteration cost of $\O{n}$. 
In other words, conversion from one GLM type to another is much
simpler than obtaining the GLM coefficients from the scratch.

\section{Binary Classification with Proper Scoring Rules}
\label{sec::binary-classification}
In this section, we assume that for $i \in [n]$, 
the response is binary $y_i\ \in \{0,1 \}$.
The binary classification problem can be described 
by the following minimization of an empirical risk
\eq{
  \underset{\beta \in \reals^p}{\minimize}\ \frac{1}{n} \sum_{i=1}^n \ell(y_i;q(\inner{x_i,\beta})),
} 
where $\ell$ and $q$ are referred to as the loss and the link functions, respectively.
There are various loss functions that are used in practice. 
Examples include log-loss, boosting loss, square loss etc (See Table \ref{tab::loss}).
As before, we constrain our analysis on the canonical links. 
The concept of canonical links for binary classification is introduced by
\cite{buja2005loss}, and it is quite similar to the generalized linear problems.

\renewcommand{\arraystretch}{1.8}
\begin{table*}[t]\footnotesize
  \caption{Common loss functions and their canonical links}
  \label{tab::loss}
  \begin{center}
    \begin{sc}
      \begin{tabular}{|l|l|c|c|}
        \hline
        Name & Loss function: $\ell(y;q)$ & Weight: $w(q)$ & Canonical link: $q(z)$\\
        \hline
        Log-loss & $-y\log(q) - (1-y)\log(1-q)$ 
                                          &$\frac{1}{q(1-q)}$ &$\frac{1}{1+\exp(-z)}$\\

        Boosting loss  & $y(q^{-1}-1)^{1/2}+(1-y)(q^{-1}-1)^{-1/2}$ 
                                          &$\frac{1}{[q(1-q)]^{3/2}}$ &$\frac{1}{2}+\frac{z/2}{2(z^2/4+1)^{1/2}}$\\

        Square loss & $y(1-q)^2+(1-y)q^2$&$1$ &$\frac{1+z}{2}$\\ 
        \hline
      \end{tabular}
    \end{sc}
  \end{center}
\end{table*}
\renewcommand{\arraystretch}{1.2}
For any given loss function, we define the partial losses $\ell_k(\cdot) = \ell(y=k; \cdot)$ 
for $k \in \{0,1 \}$. Since we have a binary response variable, 
we can write any loss in the following format
\eqn{
  \ell(y;q) =& y \ell_1(q) + (1-y) \ell_0(q),\\
  = & y\left(\ell_1(q) - \ell_0 (q) \right) + \ell_0(q).
}
The above formulation is of the form of a generalized linear problem.
Before moving forward, we recall the concept of proper scoring in binary classification, 
which is sometimes referred to as Fisher consistency.
\begin{definition}[Proper scoring rules]
  Assume that $y\sim \text{Bernoulli}(\eta)$.
  If the expected loss $\E{\ell(y,q)}$ is minimized by $q = \eta$ for all $\eta \in(0,1)$,
  we call the loss function a proper scoring rule. 
\end{definition}

The following theorem by \cite{schervish1989general} provides a methodology for 
constructing a loss function
for the proper scoring rules.
\begin{theorem}[\cite{schervish1989general}]\label{thm::schervish}
  Let $w(dt)$ be a positive measure on $(0,1)$ that is finite on interval $(\e,1-\e)$ $\forall \e>0$. Then the following defines a proper scoring rule
  \eqn{
    \ell_1(q) = \int^{1}_q (1-t)w(dt), \text{ and } \ell_0(q)= \int^q_0 t w(dt).
  }
\end{theorem}
The measure $w(dt)$ uniquely defines the loss function
(generally referred to as the weight function, since all losses can be written
as weighted average of cost weighted misclassification error 
\cite{buja2005loss, reid2010composite}).
Examples of weight functions is given in Table~\ref{tab::loss}.
The above theorem has many interesting interpretations;
one that is most useful to us is that $\ell_0^{(1)}(q) = qw(q)$. 

The notion of canonical links for proper scoring rules are introduced by \cite{buja2005loss},
which corresponds to the notion of matching loss \cite{helmbold1999relative,reid2010composite}.
The derivation of canonical links stems from the Hessian of the above minimization,
which remedies two potential problems: non-convexity and asymptotic variance inflation.
It turns out that by setting $w(q)q^{(1)}$ as constant, 
one can remedy both problems \cite{buja2005loss}. We will skip the derivation
and, without loss of generality, assume that the canonical link-loss pair satisfies $w(q)q^{(1)}=1$.
Note that any loss function has a natural canonical link. 
The following Theorem summarizes this concept.
\begin{theorem}[\cite{buja2005loss}]
  For proper scoring rules with $w>0$, there exists a canonical link function which is unique
  up to addition and multiplication by constants. Conversely, any link function is 
  canonical for a unique proper scoring rule.
\end{theorem}
The canonical link for a given loss function can be explicitly derived from 
the equation $w(q)q^{(1)}=1$. We have provided some examples in Table~\ref{tab::loss}. 
Using the definition of canonical link for proper scoring rules, 
we write the normal equations $\frac{\d }{\d \beta} \E{\ell(y,q(\inner{x,\beta}))} =0 $ as
\eqn{
  \E{xq^{(1)}(\inner{x,\beta})\ell_0^{(1)}(q(\inner{x,\beta}))} = &
  \E{yxq^{(1)}(\inner{x,\beta}) \left(\ell_0^{(1)}(q(\inner{x,\beta})) - \ell_1^{(1)}(q(\inner{x,\beta}))\right) }, \\
  \E{xq^{(1)}(\inner{x,\beta})q(\inner{x,\beta})w(q(\inner{x,\beta}))}=& \E{yq^{(1)}(\inner{x,\beta})w(q(\inner{x,\beta}))},\\
  \E{xq(\inner{x,\beta})}=& \E{yx},\\
  \Sig\beta\E{q^{(1)}(\inner{x,\beta})} = & \E{yx}.
}
The last equation provides us with the analog of the proportionality relation
we observed in generalized linear problems.
In this case, we observe that the
proportionality constant becomes $1/\E{q^{(1)}(\inner{x,\beta})}$.
Therefore, our algorithm can be used to obtain a fast training procedure 
for the binary classification problems
under canonical links.

\section{Canonicalization of the Square Loss}
\label{sec::canonicalization}
In this section, we present a method to approximate the square loss with
a canonical form. 
Using this canonical approximation, we can use the techniques developed in previous sections to gain computational benefits.
Consider a minimization problem of the following form
\eq{
  \underset{\beta}{\text{minimize}}\  \frac{1}{n}\sumton[y_i-f(\inner{x_i,\beta})]^2.
}
The above problem is commonly encountered in many machine learning tasks --
specifically, 
in the context of neural networks, the function $f$ is called the activation function.
Here, we consider a toy example to demonstrate how our methodology can be useful in
a minimization problem of the above form.

We first use Taylor series expansion around a point $\theta$
(which should be close to $\inner{x,\beta}$),
in order to approximate the function $f(z)$ with a linear function around $f(\theta)$.
This way, the square loss can be approximated with a generalized linear loss. We write
\eq{
  \min_\beta (y-f(\inner{x,\beta}))^2
  &=\min_\beta  f(\inner{x,\beta})^2- 2yf(\inner{x,\beta})  \\\nonumber
  &\approx \min_\beta \frac{f(\inner{x,\beta})^2}{2f'(\theta)}-y \inner{x,\beta} .
}
Then, we would have
\eq{
  \cgf(z) = \frac{f(z)^2}{2f'(\theta)},
}
and the proportionality relation given in previous sections would hold approximately.
The above approximation will be accurate when the activation function
is smooth around the user-specified point $\theta$.
We suggest to use $\theta=0$ since when $p$ is large and $\beta$ is well-spread, 
the inner product $\inner{x,\beta}$ should be close to its expectation $\E{\inner{x,\beta}}=0$.
This method can be used to derive proportionality
relations for GLMs with non-canonical links (conditional on link being nice),
and also may be of interest in non-convex optimization.


\section{Experiments}
\label{sec::experiments}
\begin{figure}[t]
\centering
  \includegraphics[width=.9\linewidth]{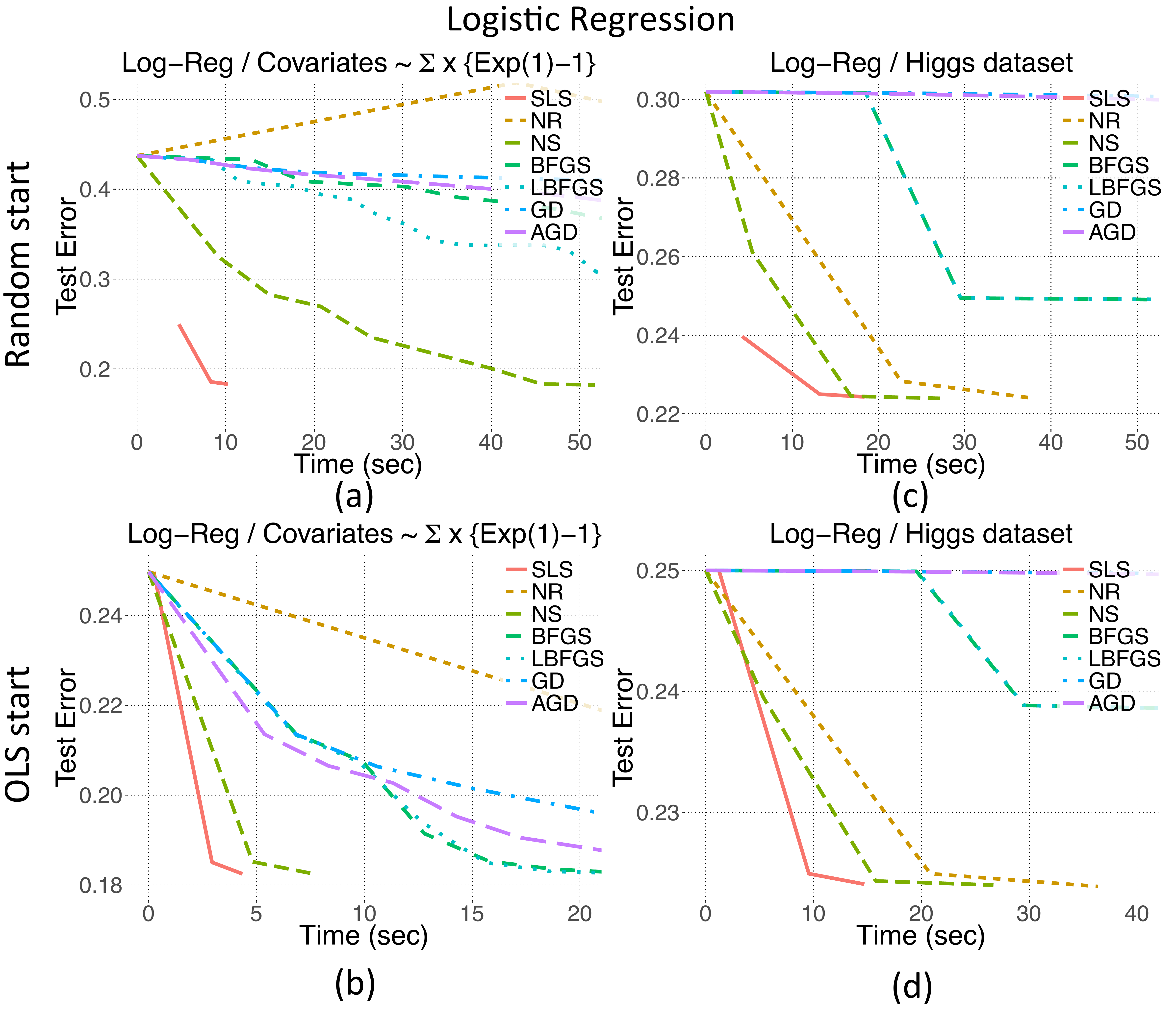}
  \caption{
    \label{fig::benchmark}
We compared the
    performance of SLS to that of MLE for the logistic regression problem on several datasets.
MLE optimization is solved by various
    optimization algorithms. 
    SLS is represented with red straight line.
    The details are provided in Table \ref{tab::results}.
  }
\end{figure}
This section contains the results of a variety of
numerical studies, which show that the Scaled Least Squares estimator 
reaches the minimum achievable test error
substantially faster than commonly used batch algorithms for finding
the MLE.  Both logistic and Poisson regression models (two types of
GLMs) are
utilized in our analyses, which are based on several synthetic and real datasets.

Below, we briefly describe the optimization algorithms 
for the MLE that were used in the experiments.

\begin{figure}[t]
  \centering
  \includegraphics[width=.9\linewidth]{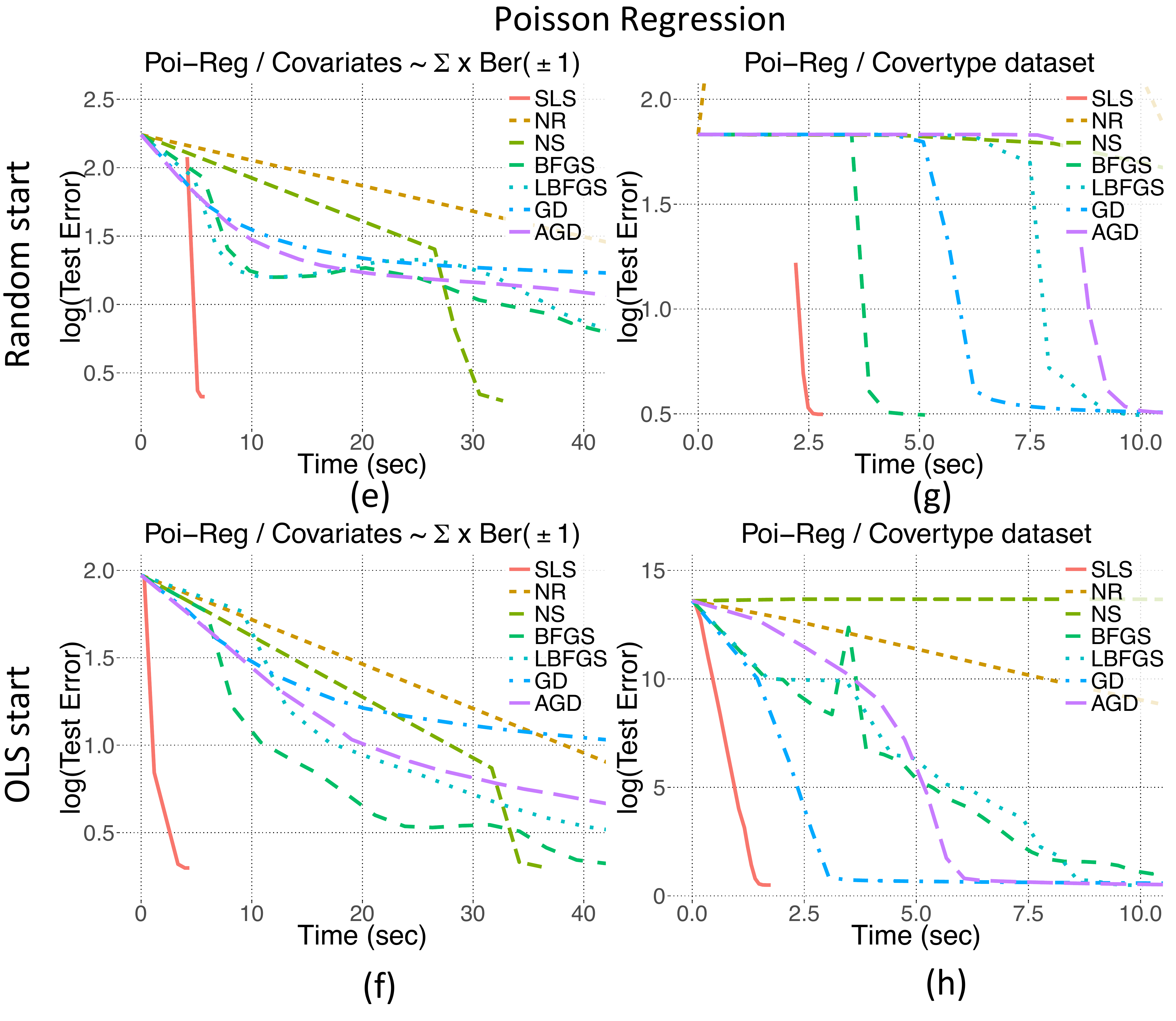}
  \caption{
    \label{fig::benchmark-1}
    We compared the
    performance of SLS to that of MLE for the Poisson regression problem on several datasets.
    MLE optimization is solved by various
    optimization algorithms. 
    SLS is represented with red straight line.
    The details are provided in Table \ref{tab::results}.
  }
\end{figure}
\begin{enumerate}
\item \textbf{Newton-Raphson (NR)}
  achieves locally quadratic convergence 
  by scaling the gradient by
  the inverse of the Hessian evaluated at the current iterate. 
  Computing the Hessian has a per-iteration cost of $\O{np^2}$,
  which makes it impractical for large-scale datasets. 
\item \textbf{Newton-Stein (NS)} is a recently proposed 
  second-order batch algorithm specifically designed for GLMs 
  \cite{erdogdu2015newton-stein, erdogdu2015newton-stein-long}. 
  The algorithm uses Stein's lemma and sub-sampling to efficiently estimate 
  the Hessian with a cost of $\O{np}$
  per-iteration, achieving near quadratic rates.
\item \textbf{Broyden-Fletcher-Goldfarb-Shanno (BFGS)} is the most popular 
  and stable quasi-Newton method \cite{nesterov2004introductory}. 
  At each iteration, the gradient is scaled by a matrix that is formed
  by accumulating information from previous iterations and gradient computations. 
  The convergence is locally super-linear with 
  a per-iteration cost of $\O{np}$.
\item \textbf{Limited memory BFGS (LBFGS)} is a variant of BFGS, 
  which uses only the recent iterates and gradients to approximate the Hessian, 
  providing significant improvement in terms of memory usage. 
  LBFGS has many variants; we use the formulation given in \cite{bishop1995neural}.
\item \textbf{Gradient descent (GD)} takes a step in 
  the opposite direction of the gradient, evaluated at the current iterate. 
  Its performance strongly depends on the condition number of the design matrix. 
  Under certain assumptions, the convergence is 
  linear with $\O{np}$ per-iteration cost.
\item \textbf{Accelerated gradient descent (AGD)} 
  is a modified version of gradient descent with 
  an additional ``momentum'' term \cite{nesterov1983method}. 
  Its per iteration cost is $\O{np}$ and
  its performance strongly depends on the smoothness of the objective function.
\end{enumerate}
For all the algorithms for computing the MLE, 
the step size at each iteration is chosen via 
the backtracking line search \cite{Boyd:2004}.

Recall that the proposed Algorithm \ref{alg::1} is composed of two
steps; the first finds an estimate of the OLS coefficients.
This up-front computation is not needed for any of the MLE algorithms
described above.  On the other hand, each of the MLE algorithms
requires some initial value for $\beta$, but no such initialization
is needed to find the OLS estimator in Algorithm \ref{alg::1}.  This
raises the question of how the MLE algorithms should be initialized,
in order to compare them fairly with the proposed method.
We consider two scenarios in our experiments:  first, we use the OLS
estimator computed for Algorithm \ref{alg::1} to initialize the MLE
algorithms; second, we use a random initial value.  

On each dataset, the main criterion for assessing the performance of
the estimators is how rapidly the minimum test error is achieved.
The test error is measured as the mean squared error of 
the estimated mean using the current parameters at each iteration on a test dataset, 
which is a randomly selected (and set-aside) 10\% portion of the entire dataset.
As noted previously, the MLE is more accurate for small $n$
(see Figure \ref{fig::mle-vs-our}).
However, in the regime considered here ($n \gg p \gg 1$), 
the MLE and the SLS perform very similarly in terms of their error
rates; for instance, on the Higgs dataset, the SLS and MLE
have test error rates of $22.40\%$
and $22.38\%$, respectively.  
For each dataset, the minimum achievable test error is set to be
the maximum of the final test errors, where the maximum is taken over
all of the estimation methods.  
Let $\Sig^{(1)}$ and $\Sig^{(2)}$ be two randomly generated covariance matrices. 
The datasets we analyzed were:  
(i) a synthetic dataset generated from a logistic regression model with
iid $\{\mbox{exponential}(1)\!-\!1\}$ predictors scaled by $\Sig^{(1)}$; 
(ii) the Higgs dataset
(logistic regression)
\cite{baldi2014searching}; 
(iii) a synthetic dataset generated
from a Poisson regression model with iid 
binary($\pm 1$) predictors
scaled by $\Sig^{(2)}$; 
(iv) the Covertype dataset (Poisson regression)
\cite{blackard1999comparative}.  

In all cases, the SLS outperformed the alternative algorithms for
finding the MLE by a large margin, in terms of computation.  Detailed
results may be found in Figures \ref{fig::benchmark} and \ref{fig::benchmark-1},
and Table \ref{tab::results}.
We provide additional experiments with different datasets 
in the Supplementary Material.

\begin{table*}[t]\scriptsize
  \caption{\small 
    \label{tab::results} Details of the experiments shown in Figures \ref{fig::benchmark} and \ref{fig::benchmark-1}.
  }
  \vspace{-.15in}
  \begin{center}
    \begin{sc}
      \begin{tabular}{l|c|c|c|c|c|c|c|c}
        Model &\multicolumn{4}{c}{Logistic regression}\vline
        &\multicolumn{4}{c}{Poisson regression} \\
        \hline
        Dataset&\multicolumn{2}{c}{$\Sig\times$\{Exp(1)-1\}}\vline
        &\multicolumn{2}{c}{Higgs {\tiny\cite{baldi2014searching}}}\vline
               &\multicolumn{2}{c}{$\Sig\times$Ber($\pm 1$)}\vline
        &\multicolumn{2}{c}{Covertype {\tiny \cite{blackard1999comparative}}} \\
        \hline
        Size &\multicolumn{2}{c}{{\scriptsize $n=6.0\times 10^5$, $p=300$}}\vline
        &\multicolumn{2}{c}{{\scriptsize $n=1.1\!\times\! 10^7$, $p=29$}}\vline
               &\multicolumn{2}{c}{{\scriptsize $\!n=6.0\!\times\! 10^5$, $p=300$\!}}\vline
        &\multicolumn{2}{c}{{\scriptsize $\!n=5.8\!\times\! 10^5$, $p=53$\!\!\!\!}}\\
        \hline
        Initialized\!\!\!\!&{\scriptsize Rnd} &{\scriptsize Ols}&
                                                                  {\scriptsize Rnd} &{\scriptsize Ols}&
                                                                                                        {\scriptsize Rnd} &{\scriptsize Ols}&
                                                                                                                                              {\scriptsize Rnd} &{\scriptsize Ols}\\
        \hline
        Plot\!&{\scriptsize (a)}&
                                  {\scriptsize (b)} &{\scriptsize (c)}&
                                                                        {\scriptsize (d)} &{\scriptsize (e)}&
                                                                                                              {\scriptsize (f)} &{\scriptsize (g)}&{\scriptsize (h)} \\
        \hline
        Method& 
                \multicolumn{8}{c}{Time in seconds / number of iterations \ 
                {\scriptsize (to reach min test error)}}\\
        \hline
        \ \ Sls  &8.34/4   & 2.94/3 & 13.18/3 &9.57/3& 5.42/5    &3.96/5   & 2.71/6   &1.66/20\\
        \ \ Nr   &301.06/6 & 82.57/3 & 37.77/3 &36.37/3& 170.28/5  &130.1/4  & 16.7/8   &32.48/18\\
        \ \ Ns   &51.69/8  & 7.8/3 & 27.11/4 &26.69/4& 32.71/5   &36.82/4  & 21.17/10 &\!\!282.1/216\!\!\!\!\\
        \ \ Bfgs &148.43/31& 24.79/8& 660.92/68 &701.9/68& 67.24/29  &72.42/26 & 5.12/7   &22.74/59\\
        \ \ Lbfgs&125.33/39& 24.61/8&\!\!6368.1/651\!\!&\!\!6946.1/670\!\!&\!224.6/106\!&\!\!357.1/88\!\!& 10.01/14 &10.05/17\\
        \ \ Gd   & 669/138 & 134.91/25 &\!\!\!100871/10101\!\!\!\!&\!\!\!141736/13808\!\!\!\!& 1711/513 &\!\!1364/374\!\!& 14.35/25 &33.58/87\\
        \ \ Agd  & 218.1/61& 35.97/12&\!\! 2405.5/251\!\! &\!\!2879.69/277\!\!& 103.3/51  &\!\!102.74/40\!\!& 11.28/15 &11.95/25\\
        \hline
      \end{tabular}
    \end{sc}
  \end{center}
  \vskip -0.1in
\end{table*}

\section{Discussion}  
\label{sec::discussion}
In this paper, 
we showed that 
the true minimizer of a generalized linear problem and the OLS estimator
are approximately proportional
under the general random design setting.
Using this relation,
we proposed a computationally efficient algorithm
for large-scale problems
that achieves the same accuracy as the empirical risk minimizer
by first estimating the OLS coefficients and
then estimating the proportionality constant through iterations that 
can attain quadratic or cubic convergence rate,
with only $\O{n}$ per-iteration cost.

We briefly mentioned that the proportionality 
between the coefficients holds even when there is regularization
in Section \ref{sec::regularization}.
Further pursuing this idea may be interesting for 
large-scale problems where regularization is crucial.
Another interesting line of research is to find similar
proportionality relations between the parameters
in other large-scale optimization problems such as 
support vector machines.
Such relations may reduce the problem complexity significantly.

%
\newpage

\bibliographystyle{amsalpha}
{\small \bibliography{bib}}
\newpage
\appendix


\section{Proof of Main Results}
In this section, we provide the details and the proofs of our technical results.
For convenience, we briefly state the following definitions.

\begin{definition}[Sub-Gaussian]
  For a given constant $\kappa$, 
  a random variable $x\in \reals$ is said to be \emph{sub-Gaussian} 
  if it satisfies
  \[
    \sup_{m \geq 1}m^{\nicefrac{-1}{2}}
    \E{|x|^m}^{1/m} \leq \kappa.
  \]
  Smallest such $\kappa$ is the sub-Gaussian norm of $x$ and 
  it is denoted by $\|x\|_\sg$. 
  Similarly, a random vector $y \in \reals^p$ is a 
  \emph{sub-Gaussian vector} if there exists a constant $\kappa'$
  such that
  \eqn{
    \sup_{v \in S^{p-1}}\| \< y, v\>\|_\sg \leq \kappa'.
  }
\end{definition}
\begin{definition}[Sub-exponential]
  For a given constant $\kappa$, 
  a random variable $x\in \reals$  is called \emph{sub-exponential} if it satisfies
  \[
    \sup_{m \geq 1}m^{-1}\E{|x|^m}^{1/m} \leq \kappa .
  \]
  Smallest such $\kappa$ is the sub-exponential norm of 
  $x$ and it is denoted by $\|x\|_\se$. 
  Similarly, a random vector $y \in \reals^p$ is a 
  \emph{sub-exponential vector} if there exists a constant $\kappa'$
  such that
  \eqn{
    \sup_{v \in S^{p-1}}\| \< y, v\>\|_\se \leq \kappa'.
  }
\end{definition}

We start with the proof of Theorem \ref{thm::population-bound}.

\begin{proof}[Proof of Theorem \ref{thm::population-bound}]
  For simplicity, we denote the whitened covariate by $w = \Sig^{-1/2}x$.
  Since $w$ is sub-Gaussian with norm $\kappa$, its $j$-th entry $w_j$
  has bounded third moment. That is,
  \eq{
    \kappa =& \sup_{\norm{u}=1}\normSG{\inner{u,w}},\\ \nonumber
    \geq & \normSG{w_j}= \sup_{m \geq 1}m^{-1/2}\E{|w_j|^m}^{1/m},\\
    \geq & \frac{1}{\sqrt{3}}\E{|w_j|^3}^{1/3},\nonumber
  }
  where in the first step, we used $u=e_j$,
  the $j$-th standard basis vector.
  Hence, we obtain a bound on the third moment, i.e, 
  \eq{
    \label{eq:bound-third-moment}
    \max_j\E{|w_j|^3} \leq 3^{3/2}\kappa^3.
  }
  Using the normal equations, we write
  \eq{\label{eq::normal-whitened-2}
    \E{yx} = \E{x \dcgf(\<x,\beta\>)} 
    =& \Sig^{1/2}\E{w \dcgf(\<w,\Sig^{1/2}\beta\>)},\\
    =&  \Sig^{1/2}\nonumber
    \E{w \dcgf(\<w,\tbeta\>)},
  }
  where we defined $\tbeta=\Sig^{1/2}\beta$.
  By multiplying both sides with $\Sig^{-1}$, we obtain
  \eq{\label{eq::ols-normal-0}
    \betaOLS =\Sig^{-1/2} \E{w \dcgf(\<w,\tbeta\>)}.
  }

  Now we define the partial sums 
  $W_{-i}=\sum_{j\neq i} \tbeta_jw_j = \inner{\tbeta, w} - \tbeta_i w_i$. 
  We will focus on the $i$-th entry of the above expectation given in
  \eqref{eq::ols-normal-0}.
  Denoting the zero biased transformation of $w_i$ conditioned on $W_{-i}$ by $w_i^*$, 
  we have
  \eq{\label{eq::zero-bias-ineq}
    \E{w_i\dcgf(\< w, \tbeta\>)} = &
    \E{\E{w_i\dcgf\left(\tbeta_i w_i+W_{-i}\right)\big | W_{-i}} },\\ \nonumber
    = & \tbeta_i\E{ \ddcgf(\tbeta_i w_i^* + W_{-i})},\\\nonumber
    = & \tbeta_i\E{ \ddcgf(\tbeta_i (w_i^*-w_i) +  \<w,\tbeta\>)},
  }
  where in the second step, we used the assumption on conditional moments.
  Let $\D$ be a diagonal matrix with diagonal entries
  $\D_{ii} = \E{ \ddcgf(\tbeta_i (w_i^*-w_i) +  \<w,\tbeta\>)}$.
  Using \eqref{eq::ols-normal-0} together with \eqref{eq::zero-bias-ineq}, 
  we obtain the equality
  \eq{\label{eq::beta-ols-zero-bias}
    \betaOLS =&\Sig^{-1/2} \D\tbeta,\\
    =&\Sig^{-1/2} \D\Sig^{1/2}\beta.\nonumber
  }

  Now, using the Lipschitz continuity assumption of the variance function,
  we have
  \eq{\label{eq::lipschitz-psi2}
    \left|\E{ \ddcgf(\tbeta_i (w_i^*-w_i) + \<w,\tbeta\>)}
      -\E{ \ddcgf( \<w,\tbeta\>)}\right| 
    \leq k |\tbeta_i| \E{|w_i^* -w_i|}.
  } 

  In the following, we will use the properties of zero-biased transformations.
  Consider the quantity
  \eq{
    r = \sup\frac{\E{|w^*_i-w_i|\big|W_{-i}}}{\E{|w_i|^3\big|W_{-i}}}
  }
  where $w^*_i$ has $w_i$-zero biased distribution (conditioned on $W_{-i}$) and
  the supremum is taken with respect to all random variables
  with mean 0, standard deviation 1 and finite third moment, 
  and $w^*_i$ is achieving the minimal $\ell_1$ coupling to $w_i$ conditioned on $W_{-i}$.
  It is shown in \cite{goldstein2007bounds} that the above bound holds for $r = 1.5$ 
  for the unconditional zero-bias transformations. 
  Here, we take a similar approach to show that the same bound holds for the conditional case as well.
  By using the triangle inequality, we have
  \eqn{
  \E{|w^*_i-w_i|\big|W_{-i}} \leq &  \E{|w^*_i|\big|W_{-i}} +   \E{|w_i|\big|W_{-i}}\\
  \leq &  \frac{1}{2}\E{|w_i|^3\big|W_{-i}} +   \E{|w_i|^3\big|W_{-i}}^{1/3}.
  }
  Since $\E{|w_i|^2\big|W_{-i}}$ is constant, it is equal to $\E{|w_i|^2} = 1$. This yields that
  the second term in the last line is upper bounded by $\E{|w_i|^3\big|W_{-i}}$. 
  Consequently, by taking expectations over both hand sides
   we obtain that
  \eqn{
    \E{|w^*_i-w_i|} \leq 1.5 \  \E{|w_i|^3} .
  }
  Then the right hand side of \eqref{eq::lipschitz-psi2} can be upper bounded by
  \eq{
    k |\tbeta_i| \E{|w_i^* -w_i|} \leq&
    r k \max_i\left\{|\tbeta_i|  \E{|w_i|^3}\right\},\\\nonumber
    \leq& 1.5k \normI{\Sig^{1/2}\beta} 3^{3/2} \kappa^3 ,\\\nonumber
    \leq& 8k \kappa^3\|\Sig^{1/2}\beta\|_\infty , \nonumber
  }
  where in the second step we used the bound on the third moment given in
  \eqref{eq:bound-third-moment}. The last inequality provides us with the
  following result,
  \eq{
    \max_i\left|\D_{ii} - \frac{1}{\c_\cgf}\right|
    \leq 8k \kappa^3\|\Sig^{1/2}\beta\|_\infty.
  }

  Finally, combining this with \eqref{eq::ols-normal-0} and \eqref{eq::beta-ols-zero-bias},
  we obtain
  \eq{
    \normI{\betaOLS - \frac{1}{\c_\cgf} \beta} = &
    \normI{\Sig^{-1/2} \D\Sig^{1/2}\beta - \frac{1}{\c_\cgf} \beta}, \\
    \nonumber
    =&\normI{\Sig^{-1/2} \left(\D -\frac{1}{\c_\cgf}\I\right)\Sig^{1/2}\beta},\\
    \nonumber
    \leq & \max_i\left|\D_{ii} - \frac{1}{\c_\cgf}\right|
    \normI{\Sig^{1/2}}\normI{\Sig^{-1/2}}\normI{\beta}^2,\\
    \nonumber
    \leq & 8k \kappa^3\rho(\Sig^{1/2})\|\Sig^{1/2}\|_\infty\frac{\tau^2 }{r^2p},
  }
  where in the last step, we used the assumption that $\beta$ is $r$-well-spread.
\end{proof}


\begin{proof}[Proof of Proposition \ref{prop::ols-rate}]
  For convenience, we denote the 
  whitened covariates with $\tx_i  = \Sig^{-1/2}x_i$.
  We have $\E{\tx_i} = 0$, $\E{\tx_i\tx_i^T} = \I$, 
  and $\normSG{\tx_i} \leq \kappa$.
  Also denote the
  sub-sampled covariance matrix with 
  $\SigH = \frac{1}{|S|}\sum_{i\in S} x_ix_i^T$,
  and its whitened version as
  $\SigT = \frac{1}{|S|}\sum_{i \in S} \tx_i\tx_i^T$.
  Further,
  define
  $\hz = \frac{1}{n}\sumton \tx_i y_i$ and
  $\zeta = \E{\tx y}$. Then, we have
  \eqn{
    \betaOLShat = \SigH^{-1}\Sig^{1/2}\hz\ \  \text{  and  }\ \ 
    \betaOLS = \Sig^{-1/2}\zeta.
  }

  For now, we work on the event that $\SigH$ is invertible.
  We will see that this event holds with very high probability.
  We write
  \eq{\label{eq::ols-main-bound}
    \norm{\Sig^{1/2}(\betaOLShat - \betaOLS)} =& 
    \norm{\Sig^{1/2}\SigH^{-1}\Sig^{1/2}\hz - \Sig^{-1/2} \zeta},\\
    \nonumber
    =& \norm{\SigT^{-1}\left\{\hz -\zeta +
        \left(\I- \Sig^{-1/2}\SigH\Sig^{-1/2}\right)\zeta\right\}},\\
    \nonumber
    \leq & \norm{\SigT^{-1}} \left\{
      \norm{\hz - \zeta} + \norm{\I- \SigT}\norm{\zeta}
    \right\},
  }
  where we used the triangle inequality and 
  the properties of the operator norm.

  For the first term on the right hand side of \eqref{eq::ols-main-bound}, 
  we write
  \eqn{
    \norm{\SigT^{-1}}
    =&  \frac{1}{ \lmin(\SigT)},\\
    \leq&\frac{1}{1 - \delta},
  }
  where we assumed that such a $\delta>0$ exists.
  In fact, when $\delta < 0.5$, we obtain a bound of $2$
  on the right hand side, which also justifies 
  the invertibility assumption of $\SigH$. 
  By Lemma \ref{lem::versh-sg} and the 
  following remark, we have with probability at least $1-2\expp{-p}$,
  \eqn{
    \norm{\SigT- \I } \leq c\sqrt{\frac{p}{|S|}},
  }
  where $c$ is a constant depending only on $\kappa$.
  When $|S| > 4c^2p$, we obtain
  \eqn{
    \abs{\lmin(\SigT) - 1} \leq 
    \norm{\SigT- \I } \leq 0.5,
  }
  where the first inequality follows from the Lipschitz property
  of the eigenvalues.

  Next, we bound the difference between $\hz$ and its expectation $\zeta$.
  We write the bounds on the sub-exponential norm
  \eq{
    \normSE{\tx y}  
    =& \sup_{\norm{v}=1}\sup_{m \geq 1} m^{-1}\E{|\inner{v,\tx}y|^m}^{1/m},\\
    \leq &\sup_{\norm{v}=1}\sup_{m \geq 1} \nonumber
    m^{-1}\E{|\inner{v,\tx}|^{2m}}^{1/2m}\E{|y|^{2m}}^{1/2m},\\
    \leq &\sup_{\norm{v}=1}\sup_{m \geq 1} \nonumber
    m^{-1/2}\E{|\inner{v,\tx}|^{2m}}^{1/2m}\sup_{m \geq 1}m^{-1/2}\E{|y|^{2m}}^{1/2m},\\
    \leq & 2\normSG{\tx}\normSG{y} = 2 \by \kappa .\nonumber
  }
  Hence, we have $\max_i\normSE{\tx_i y_i - \E{\tx_i y_i}} \leq 4\by \kappa $.
  Further, let $e_j$ denote the $j$-th standard basis,
  and notice that each entry of $\tx$ is also sub-Gaussian 
  with norm upper bounded by $\kappa$, i.e.,
  \eq{
    \kappa =\normSG{\tx} =& \sup_{\norm{u}=1}\normSG{\inner{u,\tx}},\\
    \geq& \normSG{\inner{e_j,\tx}} = \normSG{w_j}.\nonumber
  }
  Also, we can write
  \eq{
    2\gamma \kappa \geq \normSE{\tx y} = &\sup_{\norm{u}=1} 
    \sup_{m\geq 1}m^{-1}\E{\abs{\inner{u,\tx}y}^m}^{1/m},\\\nonumber
    \geq &\sup_{\norm{u}=1} \E{\abs{\inner{u,\tx}y}},\\\nonumber
    \geq &\sup_{\norm{u}=1} \E{{\inner{u,\tx}y}},\\\nonumber
    =& \sup_{\norm{u}=1} \inner{u,\zeta} = \norm{\zeta},\nonumber
  }
  where in the last step, we used the fact that dual norm of $\ell_2$ norm
  is itself.

  Next, we apply Lemma \ref{lem::vector-hoeff} to $\hz - \zeta$,
  and obtain with probability at least $1 - \expp{-p}$
  \eqn{
    \norm{\hz - \zeta} \leq c\gamma \kappa \sqrt{\frac{p}{n}}, 
  }
  whenever $n > c^2p$ for an absolute constant $c$.

  Combining the above results in \eqref{eq::ols-main-bound}, 
  we obtain with probability at least $1-3\expp{-p}$
  \eq{\label{eq::ols-bound-scaled}
    \norm{\Sig^{1/2}(\betaOLShat - \betaOLS)}
    \leq 2
    \left\{
      c_1\gamma \kappa\sqrt{\frac{p}{n}}  
      + c_2\gamma \kappa \sqrt{\frac{p}{|S|}}
    \right\}
    \leq
    \eta\sqrt{\frac{p}{|S|}}
  }
  where $\eta$ depends only on $\kappa$ and $\gamma$,
  and $|S| > \eta p$.
  Finally, we write
  \eqn{
    \norm{\betaOLShat - \betaOLS} \leq 
    &\lmin^{-1/2}\norm{\Sig^{1/2}(\betaOLShat - \betaOLS)},\\
    \leq& \eta
    \lmin^{-1/2}\sqrt{\frac{p}{|S|}},
  }
  with probability at least $1- 3\expp{-p}$, whenever $|S| > \eta p$.
\end{proof}

The following lemma -- combined with the Proposition \ref{prop::ols-rate} -- 
provides the necessary tools to 
prove Theorem \ref{thm::main-bound}.

\begin{lemma}\label{lem::concentration-c}
  For a given function $\ddcgf$ that is Lipschitz continuous
  with $k$, and uniformly bounded by $b$,
  we define the function $f:\reals \times \reals^p \to \reals$ as
  \eqn{
    f(c, \beta) = c\ \E{\ddcgf(\inner{x,\beta}c )},
  } 
  and its empirical counterpart as
  \eqn{
    \hat{f}(c,\beta) = c \ \frac{1}{n} \sum_{i=1}^n \ddcgf(\inner{x_i,\beta}c).
  }
  Assume that for some $\delta, \bar{c} >0$,  we have
  $f(\bar{c},\betaOLS ) \geq 1+\delta$.
  Then, $\exists c_\cgf >0 $ satisfying the equation
  \eqn{
    1 = f(c_\cgf,\betaOLS).
  }
  Further, 
  assume that for some $\tilde{\delta}>0$, 
  we have $\delta =  \tilde{\delta} \sqrt{p}$,
  and 
  $n$ and $|S|$ sufficiently large, i.e.,
  \eqn{
    \minn{\frac{n}{\log(n)},|S|} > K^2/\tilde{\delta}^2
  }
  for $K=\eta\bar{c}\maxx{b+\kappa/\tmu, k \bar{c} \kappa}$.
  Then, with probability $1-5\expp{-p}$,
  there exists a constant $\hc_\cgf \in (0,\bar{c})$ satisfying the equation
  \eqn{
    1 = \hc_\cgf\ \frac{1}{n} \sumton \ddcgf(\inner{x_i,\betaOLShat}\hc_\cgf).
  }
  Moreover, if the derivative of $z \to f(z,\betaOLS)$
  is bounded below in absolute value (i.e. does not change sign) 
  by $\upsilon >0$ in the interval $z \in [0,\bar{c}]$,
  then
  with probability $1-5\expp{-p}$, we have
  \eqn{
    \abs{\hc_\cgf - c_\cgf} \leq 
    C  \sqrt{\frac{p}{\minn{n/\logg{n},|S|}}},
  }
  where $C = K/\upsilon$.
\end{lemma}

\begin{proof}[Proof of Lemma \ref{lem::concentration-c}]
  First statement is obvious. We notice that $f(c,\betaOLS)$ is a continuous function in
  its first argument with $f(0,\betaOLS) =0$ and $f(\bar{c},\betaOLS) \geq 1+\delta$.
  Hence, there exists $c_\cgf>0$ such that $f(c_\cgf,\betaOLS) = 1$.
  If there are many solutions to the above equation, we choose
  the one that is closest to zero. The condition on the derivative will 
  guarantee the uniqueness of the solution.

  Next, we will show the existence of $\hc_\cgf$ using
  a uniform concentration given by Lemma \ref{lem::concentration-g}.
  Define the ellipsoid centered around $\betaOLS$ with radius $\delta$,
  $$\B_\Sig^\delta(\betaOLS) = 
  \left\{\beta : \big\|\Sig^{1/2}(\beta - \betaOLS)\big \|_2 \leq \delta \right\},$$
  and the event $\Ev$ that $\betaOLShat$ falls into $\Bols$, i.e.,
  \eqn{
    \Ev = \left\{ \betaOLShat \in \Bols\right\}. 
  }
  By Proposition \ref{prop::ols-rate} and the inequality given in 
  \eqref{eq::ols-bound-scaled}, whenever
  $|S| > \eta p\maxx{1,\eta/\delta^{2}}$, we obtain
  \eqn{
    \P{\Ev^C} \leq 3\expp{-p},
  }
  where $\Ev^C$ denotes the complement of the event $\Ev$,
  and $\eta$ is a constant depending only on $\kappa$ and $\gamma$.
  For any $c \in [0, \bar{c}]$, on the event $\Ev$, we have
  \eqn{
    \abs{\hat{f}(c,\betaOLShat) - f(c,\betaOLShat)} 
    \leq \sup_{\beta \in \Bols}\abs{\hat{f}(c,\beta) - f(c,\beta)}.
  }
  Hence, we obtain the following inequality
  \eqn{
    \P{\sup_{c \in [0,\bar{c}]}\abs{\hat{f}(c,\betaOLShat) - f(c,\betaOLShat)} >\e }\leq&
    \P{\sup_{c \in [0,\bar{c}]}
      \abs{\hat{f}(c,\betaOLShat) - f(c,\betaOLShat)} >\e ;\Ev}+ \P{\Ev^C},\\
    \leq& \P{\sup_{c \in [0,\bar{c}]}
      \sup_{\beta \in \Bols}\abs{\hat{f}(c,\beta) - f(c,\beta)} >\e}
    + 3\expp{-p}. 
  } 
  In the following, we will use Lemma \ref{lem::concentration-g} 
  for the first term in the last line above. 
  Denoting by $w$, the whitened covariates,
  we have $\inner{x,\beta} = \inner{w,\Sig^{1/2}\beta}$.
  Therefore,
  \eqn{
    &\sup_{c \in [0,\bar{c}]}
    \sup_{\beta \in \Bols}\abs{\hat{f}(c,\beta) - f(c,\beta)}\\ 
    &\leq
    \bar{c}\sup_{c \in [0,\bar{c}]}\sup_{\beta \in \Bols} 
    \abs{\frac{1}{n}\sum_{i=1}^n\ddcgf(\inner{w_i,\Sig^{1/2}\beta}c)
      -\E{\ddcgf(\inner{w,\Sig^{1/2}\beta}c)}}.
  }
  Next, define the ball centered around $\betaOLStilde=\Sig^{1/2}\betaOLS$,
  with radius $\delta$ as $\Bolst = \Sig^{1/2}\Bols$.
  We have $\beta \in \Bols$ if and only if $\Sig^{1/2}\beta \in \Bolst$.
  Then, the right hand side of the above inequality can be written as
  \eqn{
    &\bar{c}\sup_{c \in [0,\bar{c}]}\sup_{\beta \in \Bolst} 
    \abs{\frac{1}{n}\sum_{i=1}^n\ddcgf(\inner{w_i,\beta}c)
      -\E{\ddcgf(\inner{w,\beta}c)}},\\
    &=\bar{c}\sup_{\beta \in \B_{\bar{c}\delta}(\tilde{\beta}^{\sf ols})} 
    \abs{\frac{1}{n}\sum_{i=1}^n\ddcgf(\inner{w_i,\beta})
      -\E{\ddcgf(\inner{w,\beta})}}.
  }
  Then, by Lemma \ref{lem::concentration-g},
  we obtain 
  \eq{
    \label{eq::bound-fhat-f}
    \P{\sup_{c \in [0,\bar{c}]}
      \abs{\hat{f}(c,\betaOLShat) - f(c,\betaOLShat)} >
      c'\bar{c}(b+\kappa/\tmu)
      \sqrt{\frac{p}{n/\logg{n}}} }
    \leq 5\expp{-p}
  }
  whenever $np > 51\maxx{\chi,\chi^{-1}}$ where 
  $\chi = (b+\kappa/\tmu)^2/(c'\delta^2k^2\bar{c}^2\tmu^2)$.

  Also, by the Lipschitz condition for $\ddcgf$, we have
  for any $c \in [0,\bar{c}]$, and $\beta_1, \beta_2$, 
  \eqn{ 
    \abs{f(c, \beta_1) - f(c, \beta_2)}
    \leq &k c^2  \E{\abs{\inner{w,\Sig^{1/2}(\beta_1 - \beta_2)}}} \\
    \leq & k \bar{c}^2 \kappa \norm{\Sig^{1/2}(\beta_1 - \beta_2)}.
  } 
  Applying the above bound for 
  $\beta_1 = \betaOLShat$ and $\beta_2 = \betaOLS$,
  we obtain with probability $1-3\expp{-p}$
  \eq{\label{eq::lipschitz-f}
    \abs{f(c, \betaOLShat) - f(c, \betaOLS)}
    \leq  \eta k \bar{c}^2 \kappa\sqrt{\frac{p}{|S|}},
  }
  where the last step follows from Proposition \ref{prop::ols-rate}
  and the inequality given in 
  \eqref{eq::ols-bound-scaled}.

  Combining this with the previous bound, 
  and taking into account that $\mu = \tmu\sqrt{p}$,
  for any $c \in [0,\bar{c}]$,
  with probability $1-5\expp{-p}$, we obtain
  \eqn{
    \abs{\hat{f}(c,\betaOLShat) - f(c,\betaOLS)} 
    \leq&
    c'\bar{c}(b+\kappa/\tmu)
    \sqrt{\frac{p}{n/\logg{n}}} 
    + \eta k \bar{c}^2\kappa\sqrt{\frac{p}{|S|}}\\
    \leq& K \sqrt{\frac{p}{\minn{n/\logg{n},|S|}}} 
  }
  where $K = \eta\bar{c}\maxx{b+\kappa/\tmu, k \bar{c}\kappa }$.
  Here, $\eta$ depends only on $\kappa$ and $\gamma$.

  In particular, for $c = \bar{c}$ we observe that
  \eqn{
    \hat{f}(\bar{c},\betaOLShat) \geq& f(\bar{c},\betaOLS) -
    K \sqrt{\frac{p}{\minn{n/\logg{n},|S|}}} \\
    \geq & 1 + \delta -
    K \sqrt{\frac{p}{\minn{n/\logg{n},|S|}}}.
  }
  Therefore, 
  for sufficiently large $n$ and $|S|$ satisfying
  \eqn{
    \minn{\frac{n}{\log(n)},|S|} > K^2/\tilde{\delta}^2
  }
  we obtain $\hat{f}(\bar{c},\betaOLShat)>1$.
  Since this function is continuous and $\hat{f}(0,\betaOLShat)=0$,
  we obtain the existence of $\hc_\cgf \in [0, \bar{c}]$ 
  with probability at least $1-5\expp{-p}$.

  Now, since $\hc_\cgf$ and $c_\cgf$ satisfy the equations
  $
  \hat{f}(\hc_\cgf, \betaOLShat) = f(c_\cgf,\betaOLS) =  1
  $
  (with high probability),
  by the inequality given in \eqref{eq::bound-fhat-f}, 
  with probability at least $1-5\expp{-p}$, we obtain
  \eqn{
    \abs{1 - f(\hc_\cgf, \betaOLShat)} = &
    \abs{\hat{f}(\hc_\cgf,\betaOLShat) - f(\hc_\cgf, \betaOLShat)}\\
    \leq & c'\bar{c}(b+\kappa/\tmu)\sqrt{\frac{p}{n/\log(n)}}.
  }

  Also, by the same argument in \eqref{eq::lipschitz-f}, 
  and Proposition \ref{prop::ols-rate}, we get 
  \eqn{
    \abs{f(\hc_\cgf, \betaOLShat) - f(\hc_\cgf, \betaOLS)}
    \leq& k  \bar{c}^2 \kappa \norm{\Sig(\betaOLShat - \betaOLS)}\\
    \leq& \eta k\bar{c}^2 \kappa\sqrt{\frac{p}{|S|}}.
  }

  Now, using the Taylor's series expansion of $c \to f(c,\betaOLS)$ around $c_\cgf$,
  and the assumption on the derivative of $f$ with respect to its
  first argument, we obtain
  \eqn{
    \upsilon \abs{\hc_\cgf - c_\cgf} 
    \leq& \abs{f(\hc_\cgf, \betaOLS) - f(c_\cgf, \betaOLS)}\\
    \leq & \abs{f(\hc_\cgf, \betaOLS) -f(\hc_\cgf, \betaOLShat)}
    + \abs{f(\hc_\cgf, \betaOLShat) - 1}\\
    \leq &  \eta k\bar{c}^2 \kappa \sqrt{\frac{p}{|S|}} 
    + c'\bar{c}(b+\kappa/\tmu)\sqrt{\frac{p}{n/\log(n)}}\\
    \leq & K  \sqrt{\frac{p}{\minn{n/\logg{n},|S|}}}
  }
  with probability at least $1-5\expp{-p}$. 
  Here, the constant $K$ is the same as before
  $$K = \eta\bar{c}\maxx{b+\kappa/\tmu,k \bar{c}\kappa }.$$

\end{proof}

\begin{proof}[Proof of Theorem \ref{thm::main-bound}]
  We have
  \eq{\label{eq::main-ineq-b}
    \normI{\betaOURhat - \betaGLM}  =& \normI{\hc_\cgf \betaOLShat - \betaGLM},\\
    \nonumber
    \leq & \normI{c_\cgf \betaOLS - \betaGLM} 
    + \normI{\hc_\cgf \betaOLShat  - c_\cgf \betaOLS},
  }
  where we used the triangle inequality for the $\ell_\infty$ norm.
  The first term on the right hand side can be bounded using
  Theorem \ref{thm::population-bound}. We write
  \eq{ 
    \normI{c_\cgf \betaOLS - \betaGLM} \leq \eta_1\ \frac{1}{p},
  }
  for $\eta_1=8k \bar{c}\kappa^3\rho(\Sig^{1/2})\|\Sig^{1/2}\|_\infty(\tau/r)^2$.

  For the second term, we write
  \eq{\label{eq::main-ineq-c}
    \normI{\hc_\cgf \betaOLShat  - c_\cgf \betaOLS} =& 
    \normI{\hc_\cgf \betaOLShat \pm \hc_\cgf \betaOLS - c_\cgf \betaOLS},\\
    \nonumber
    \leq &\normI{\hc_\cgf \betaOLShat - \hc_\cgf \betaOLS} 
    +\normI{\hc_\cgf \betaOLS - c_\cgf \betaOLS},\\
    \nonumber
    \leq & \abs{\hc_\cgf} \normI{\betaOLShat - \betaOLS}
    + \abs{\hc_\cgf - c_\cgf} \normI{\betaOLS},
  }
  where the first step follows from triangle inequality.
  By Lemma \ref{lem::concentration-c}, for sufficiently large $n$ and $|S|$, 
  with probability $1-5\expp{-p}$,
  the constant $\hc_\cgf$ exists and it is in the interval $ (0,\bar{c}]$.
  By the same lemma, with probability $1-5\expp{-p}$,
  we have
  \eq{
    \abs{\hc_\cgf - c_\cgf} \leq \eta_4 \sqrt{\frac{p}{\minn{n/\logg{n},|S|}}},
  }
  where $\eta_4 = \eta' \upsilon^{-1}\bar{c}\maxx{b+\kappa/\tmu, k \bar{c}\kappa }$,
  for some constant $\eta'$ depending on the sub-Gaussian norms $\kappa$ and $\gamma$.

  Also, by the norm equivalence and 
  Proposition \ref{prop::ols-rate}, we have with probability $1-3\expp{-p}$
  \eq{
    \normI{\betaOLShat - \betaOLS} 
    \leq &\eta_3\sqrt{\frac{p}{|S|}},
  }
  for $\eta_3 = \eta''\lmin^{-1/2}$, 
  where $\eta''$ is constant depending only on 
  $\gamma$ and $\kappa$.

  Finally, combining all these inequalities with 
  the last line of \eqref{eq::main-ineq-b},
  we have
  with probability $1-5\expp{-p}$,
  \eq{
    \normI{\betaOURhat - \betaGLM}&\leq
    \eta_1 \frac{1}{p}+
    \eta_3\bar{c}\ \sqrt{\frac{p}{|S|}} + 
    \eta_4\normI{\betaOLS} \sqrt{\frac{p}{\min\{n/\log(n),|S|\}}},\\
    &\leq \eta_1 \frac{1}{p}+\left(\eta_3 \bar{c}+ \eta_4 \normI{\betaOLS} \right)
    \sqrt{\frac{p}{\minn{n/\logg{n},|S|}}},\nonumber \\\nonumber
    =&\eta_1 \frac{1}{p} + \eta_2 \sqrt{\frac{p}{\minn{n/\logg{n},|S|}}},
  }
  where
  \eq{
    \eta_1 =& 8k \bar{c}\kappa^3\rho(\Sig^{1/2})\|\Sig^{1/2}\|_\infty(\tau/r)^2\\
    \eta_2 =& \eta_3 \bar{c}+ \eta_4 \normI{\betaOLS},\nonumber\\
    =&  \eta\bar{c}\lmin^{-1/2} 
    \left(1+ \upsilon^{-1}\lmin^{1/2} \|\betaOLS\|_\infty\maxx{(b+k/\tmu),
        k \bar{c}\kappa} \right).\nonumber
  }
\end{proof}


\begin{proof}[Proof of Corollary \ref{cor::lasso}]
  The normal equations for the lasso minimization yields
  \eqn{
    \E{xx^T}\betaLASSO_\lambda -\betaOLS + \lambda s = 0,
  }
  where $s \in \partial \normO{\betaLASSO_\lambda}$.
  It is well-known that under the orthogonal design 
  where the covariates have i.i.d. entries,
  the above equation reduces to
  \eqn{
    \soft(\betaOLS; \lambda) = \betaLASSO_\lambda,
  }
  where $\soft(\ \cdot \ ; \lambda)$ denotes the soft thresholding
  operator at level $\lambda$.
  For any $\beta \in \reals^p$, 
  let $\supp(\beta)$ denote the support of $\beta$, i.e.,
  the set $\{i\in [p] : \beta_i \neq 0 \}$.
  We have
  \eqn{
    \supp(\betaLASSO_\lambda) &= \{ i \in [p]: \betaLASSO_{\lambda, i} \neq 0\},\\
    &= \{i \in [p]: |\betaOLS_i| >\lambda \}
  }
  By Theorem \ref{thm::population-bound}, we have
  \eqn{
    |\betaOLS_i| \leq \frac{1}{\c_\cgf}|\betaGLM_i| 
    + \frac{\eta}{\vert\supp(\betaGLM)\vert},
  }
  which implies that
  \eqn{
    \supp(\betaLASSO_\lambda) \subset\left \{i\in [p]: 
      \frac{1}{\c_\cgf}|\betaGLM_i| + 
      \frac{\eta}{\vert\supp(\betaGLM)\vert} > \lambda\right\}.
  }
  Hence, whenever $\lambda > \eta/\vert\supp(\betaGLM)\vert$, we have
  \eqn{
    \supp(\betaLASSO_\lambda) \subset
    \supp(\betaGLM).
  }
  Further, we have by Theorem \ref{thm::population-bound}
  \eqn{
    \frac{1}{\c_\cgf}|\betaGLM_i| \leq |\betaOLS_i|  + \frac{\eta}{\vert\supp(\betaGLM)\vert}.
  }
  Hence, whenever $|\betaGLM_i| > \c_\cgf\left( \lambda +
    \eta/\vert\supp(\betaGLM)\vert\right)$, we get $|\betaOLS_i| > \lambda$.
  If this condition is satisfied for any entry in the support of $\betaGLM$,
  the corresponding lasso coefficient will be non-zero. Therefore, we get
  \eqn{
    \supp(\betaGLM) \subset
    \supp(\betaLASSO_\lambda)
  }
  under this assumption. Combining this with the previous result, we conclude the proof.
\end{proof}



\begin{figure}[t]
  \centering
  \includegraphics[width=.9\linewidth]{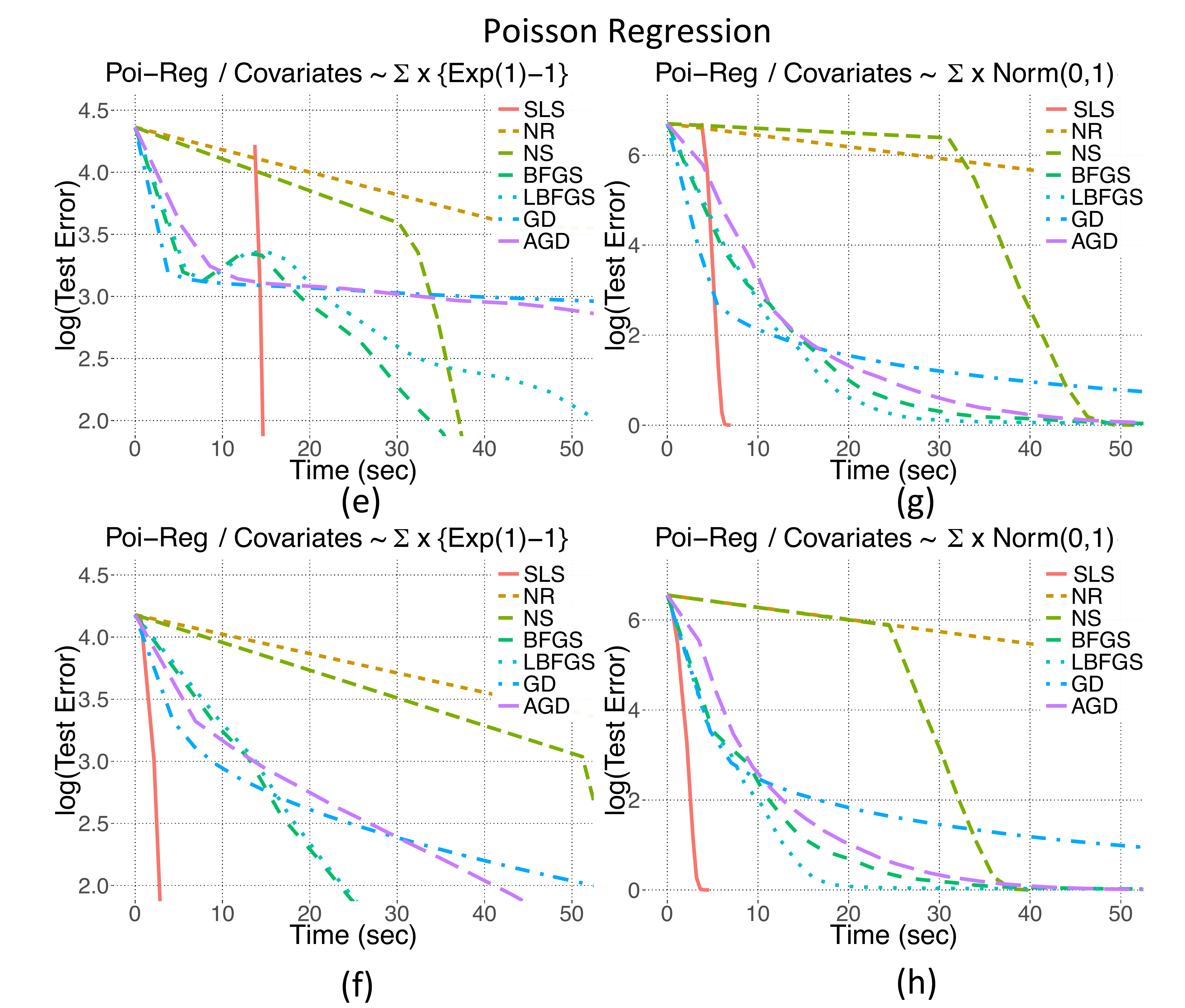}
  \caption{
    \label{fig::benchmark-2}
    Additional experiments comparing the performance of SLS to that of 
    MLE obtained with various
    optimization algorithms on several datasets. 
    SLS is represented with red straight line.
    The details are provided in Table \ref{tab::results-2}
  }
\end{figure}
\section{Additional Experiments}\label{sec::additional-experiments}
In this section, we provide additional experiments.
The overall setting is the same as Section \ref{sec::experiments}.
The only difference is that we change the sampling distribution of the
datasets, which are stated in the title of each plot.
As in Section \ref{sec::experiments}, SLS estimator
outperforms its competitors by a large margin
in terms of the computation time.

The results are provided in Figures \ref{fig::benchmark-2} and \ref{fig::benchmark-3},
and Table \ref{tab::results-2}.
\begin{figure}[h]
  \centering
  \includegraphics[width=.9\linewidth]{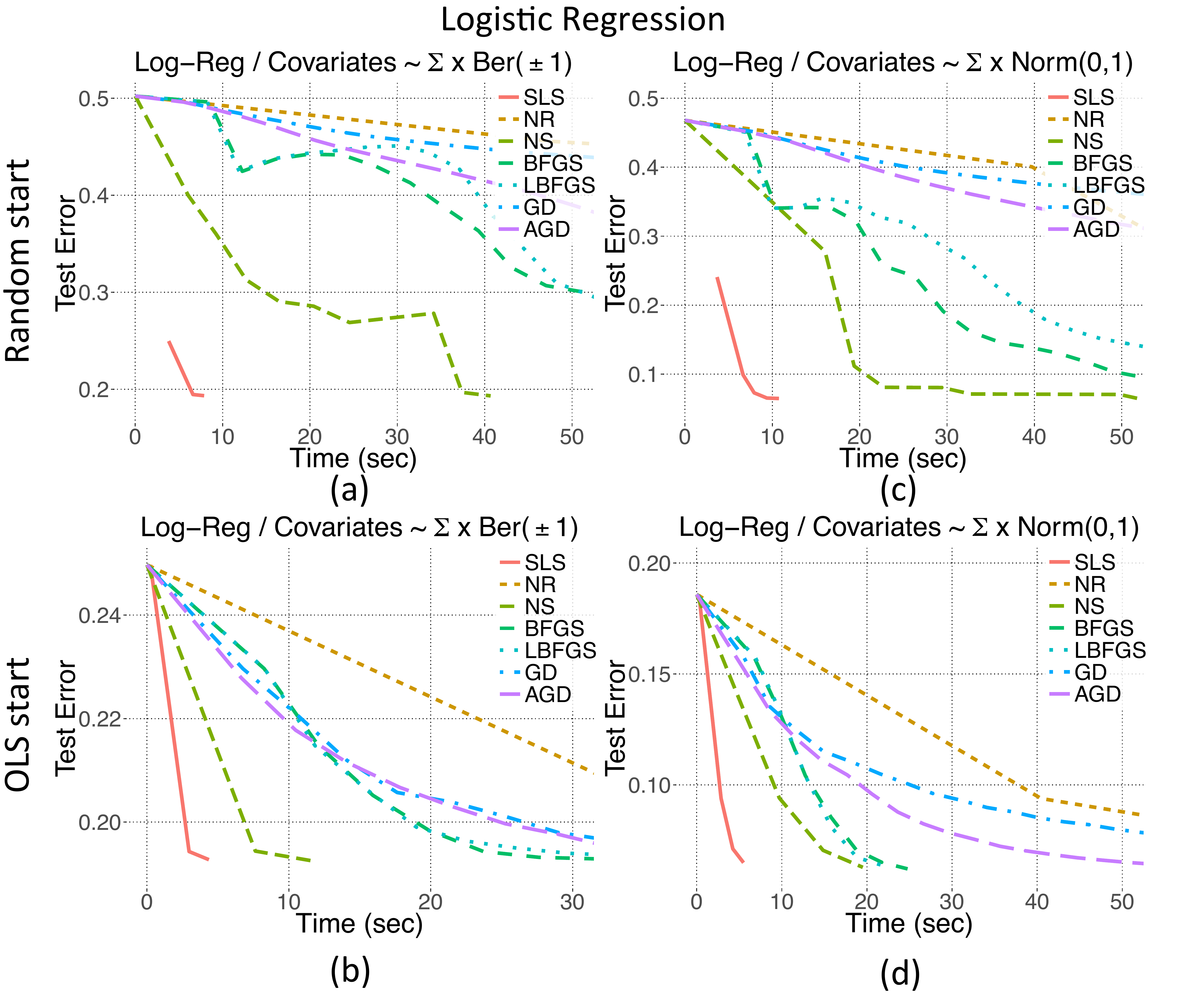}
  \caption{
    \label{fig::benchmark-3}
    Additional experiments comparing the performance of SLS to that of 
    MLE obtained with various
    optimization algorithms on several datasets. 
    SLS is represented with red straight line.
    The details are provided in Table \ref{tab::results-2}
  }
\end{figure}

\begin{table*}[h]\scriptsize
  \caption{\label{tab::results-2} 
    Details of the experiments shown in Figures \ref{fig::benchmark-2} and  \ref{fig::benchmark-3}.
  }
  \begin{center}
    \begin{sc}
      \begin{tabular}{l|c|c|c|c|c|c|c|c}
        Model &\multicolumn{4}{c}{Logistic Regression}\vline
        &\multicolumn{4}{c}{Poisson Regression} \\
        \hline
        Dataset
              &\multicolumn{2}{c}{$\Sig\times$Ber($\pm 1$)}\vline
        &\multicolumn{2}{c}{$\Sig\times$Norm(0,1)}\vline
              &\multicolumn{2}{c}{$\Sig\times$\{Exp(1)-1\}}\vline
        &\multicolumn{2}{c}{$\Sig\times$Norm(0,1)} \\
        \hline
        Size &\multicolumn{2}{c}{{\scriptsize $n=6.0\!\times\! 10^5$, $p\!=\!300$}}\vline
        &\multicolumn{2}{c}{{\scriptsize $n=6.0\!\times\! 10^5$, $p\!=\!300$}}\vline
              &\multicolumn{2}{c}{{\scriptsize $n=6.0\!\times\! 10^5$, $p\!=\!300$}}\vline
        &\multicolumn{2}{c}{{\scriptsize $n=6.0\!\times\! 10^5$, $p\!=\!300$}}\\
        \hline
        Initialize\!\!&{\scriptsize Rnd} &{\scriptsize Ols}&
                                                             {\scriptsize Rnd} &{\scriptsize Ols}&
                                                                                                   {\scriptsize Rnd} &{\scriptsize Ols}&
                                                                                                                                         {\scriptsize Rnd} &{\scriptsize Ols}\\
        \hline
        Plot\!&{\scriptsize (a)}&
                                  {\scriptsize (b)} &{\scriptsize (c)}&
                                                                        {\scriptsize (d)} &{\scriptsize (e)}&
                                                                                                              {\scriptsize (f)} &{\scriptsize (g)}&{\scriptsize (h)} \\
        \hline
        Method$\downarrow$\!\!& 
                                \multicolumn{8}{c}{Time in Seconds / Number of Iterations \ 
                                {\scriptsize (to reach min test error)}}\\
        \hline
        \ \ Sls  &6.61/3   &2.97/3  &9.38/5    & 4.25/4   &14.68/4  &2.99/4   &6.66/10  &4.13/10\\
        \ \ Nr   &222.21/6 &84.08/3 &186.33/6  & 115.76/4 &218.1/6  &218.9/4  &364.63/9 &363.4/9\\
        \ \ Ns   &40.68/10 &11.57/3 &53.06/9   & 19.52/4  &39.22/6  &59.61/4  &51.48/10 &39.8/10\\
        \ \ Bfgs &125.83/33&35.41/9 &155.3/48  & 24.78/8  &46.61/20 &48.71/12 &92.84/36 &74.22/38\\
        \ \ LBfgs&142.09/38&44.41/12&444.62/143& 21.79/7  &96.53/39 &50.56/12 &296.4/111&228.1/117\\
        \ \ Gd   &409.9/134&79.45/22&1773.1/509& 135.62/44&569.1/211&124.31/48&792.3/344&1041.1/366\\
        \ \ Agd  &177.3/159&43.76/12&359.56/95 & 53.73/18 &157.9/57 &63.16/16 &74.74/32 &62.21/32\\
        \hline
      \end{tabular}
    \end{sc}
  \end{center}
\end{table*}
\newpage
\section{Auxiliary Lemmas}

\begin{lemma}[Sub-exponential vector concentration]
  \label{lem::vector-hoeff}
  Let $x_1,x_2,..., x_n$ be independent centered sub-exponential
  random vectors with $\max_i\|x_i\|_\se = \kappa$.
  Then we have
  \eq{\label{eq::sub-exp-1}
    \P{\norm{\frac{1}{n}\sumton x_i} > c\kappa \sqrt{\frac{p}{n}}}
    \leq \expp{-p}.
  }
  whenever $n > 4c^2p$ for an absolute constant $c$.
\end{lemma}
\begin{proof}[Proof of Lemma \ref{lem::vector-hoeff}]
  For a vector $z \in \reals^p$,
  we have $\norm{z} = \sup_{\norm{u}=1} \inner{u,z}$ since the dual of $\ell_2$ norm is itself.
  Therefore, we write
  \eqn{
    \P{\norm{\frac{1}{n}\sumton x_i} > t}
    =& \P{\sup_{\norm{u}=1} \frac{1}{n}\sumton \inner{u,x_i} > t }.
  }
  Now, let $\N_\e$ be an $\e$-net over $\S^{p-1}= \{u\in\reals^p:\ \norm{u}=1 \}$,
  and observe that
  \eqn{
    \max_{u \in \N_\e}\  \< u,x\> \geq &(1-\e)\sup_{\norm{u}=1}\<u,x\>,\\
    = & (1-\e) \|x\|_2,
  }
  with $\abs{\N_\e}\leq (1+2/\e)^p$.
  Hence, we may write
  \eqn{
    \P{\sup_{\norm{u}=1} \frac{1}{n}\sumton \inner{u,x_i} > t }
    \leq & \P{\max_{u \in \N_\e} \frac{1}{n}\sumton \inner{u,x_i} > {t}(1-\e)},\\
    \leq & \abs{\N_\e}\P{ \frac{1}{n}\sumton \inner{u,x_i} > t(1-\e)}.\\
  }
  For any $u \in \S^{p-1}$, we have $\normSE{\inner{u,x_i}}\leq \kappa$.
  Then, by the Bernstein-type inequality 
  for sub-exponential
  random variables \cite{vershynin2010introduction}, we have
  \eqn{
    \P{ \frac{1}{n}\sumton \inner{u,x_i} > {t}(1-\e)}
    \leq \expp{-cn\min\left\{\frac{t^2(1-\e)^2}{\kappa^2},\frac{t(1-\e)}{\kappa} \right\}},
  }
  for an absolute constant $c$.
  Therefore, the probability on the left hand side of \eqref{eq::sub-exp-1} 
  can be bounded by
  \eqn{
    \left(1+\frac{2}{\e}\right)^p \expp{-cn\frac{t^2(1-\e)^2}{\kappa^2}}
    = & \expp{-cn\frac{t^2(1-\e)^2}{\kappa^2} + p \log\left(1+\frac{2}{\e}\right)},
  }
  whenever $t < \kappa / (1-\e)$.
  Choosing $\e = 0.5$ and for an absolute constant $c'> 3.24/c$ and letting 
  $$t = c'\kappa \sqrt{\frac{p}{n}},$$
  we conclude the proof.
\end{proof}

\begin{lemma}\label{lem::concentration-g}
  Let $B(\tbeta)$ denote the ball centered around $\tbeta$ with radius $\delta$, 
  i.e., $$B(\tbeta) = \left\{\beta : \big\|\beta - \tbeta\big\|_2 \leq \delta \right\}.$$

  For $i=1,...,n$, let $x_i\in \reals^p$ be i.i.d. 
  centered sub-Gaussian random vectors
  with norm bounded by $\kappa$ and $\E{\norm{x}} = \tmu \sqrt{p}$. 
  Given a function $g:\reals \to \reals$ that is 
  uniformly bounded by $b >0 $, 
  and Lipschitz continuous with $k$,
  \eqn{
    \P{ \sup_{\beta \in B} \abs{ \frac{1}{n}\sum_{i=1}^n g(\<x_i,\beta \>) 
        - \E{g(\< x,\beta\>)}  } > c(b+\kappa/\tmu)\sqrt{\frac{ p}{n/\log(n)}} } \leq 2\expp{-p},
  }
  whenever 
  $np > 51\max\{\chi,\chi^{-1} \}$ for 
  $\chi = (b+\kappa/\tmu)^2/(c\delta^2k^2\tmu^2)$. Above,
  $c$
  is an absolute constant.
\end{lemma}

\begin{proof}[Proof of Lemma \ref{lem::concentration-g}]
  Let $\E{\| x\|_2} = \mu = \tmu \sqrt{p}$ and
  for $\e >0$, $\beta \in B(\tbeta)$
  and $w \in \reals^p$
  define the bounding functions
  \eqn{
    l_\beta(w) =& g(\< w,\beta\>) - \e{\| w\|_2}/{4\mu},\\
    u_\beta(w)=& g(\< w,\beta\>) + \e{\| w\|_2}/{4\mu}.
  }
  Let $\N_\Delta$ be a net over $B(\tbeta)$ in the sense that for any $\beta_1 \in B(\tbeta)$,
  $\exists \beta_2 \in \N_\Delta$ such that $\norm{\beta_1 - \beta_2} \leq \Delta$. 
  We fix $\Delta_* = \epsilon /(4k \mu)$ and write 
  $\forall \beta_1 \in B$, $\exists \beta_2\in \N_{\Delta_*}$,
  \begin{enumerate}{}
  \item an upper bound of the form:
    \eqn{
      g(\inner{w,\beta_1}) \leq& g(\inner{w,\beta_2}) + k\abs{\inner{w,\beta_1-\beta_2}},\\
      \leq& g(\inner{w,\beta_2}) + k \norm{w} \Delta_*, \\
      = & u_{\beta_2}(w),
    }
  \item and a lower bound of the form:
    \eqn{
      g(\inner{w,\beta_1}) \geq& g(\inner{w,\beta_2}) - k\abs{\inner{w,\beta_1-\beta_2}},\\
      \geq& g(\inner{w,\beta_2}) - k \norm{w} \Delta_*, \\
      = & l_{\beta_2}(w),
    }
  \end{enumerate}
  where the second steps in the above inequalities follow from 
  the Cauchy-Schwarz inequality.
  These functions are called \emph{bracketing functions}
  in the context of empirical process theory.

  Hence, we can write that 
  $\forall \beta_1 \in B(\tbeta)$, $\exists \beta_2 \in \N_{\Delta_*}$
  such that
  \eqn{
    \frac{1}{n}\sum_{i=1}^n l_{\beta_2}(x_i) - \E{l_{\beta_2}(x)} -\e/2 
    & \leq 
    \frac{1}{n}\sum_{i=1}^n g(\<x_i,\beta_1 \>) - \E{g(\< x,\beta_1\>)},\\
    & \leq  \frac{1}{n}\sum_{i=1}^n u_{\beta_2}(x_i) - \E{u_{\beta_2}(x)} +\e/2
    .
  }

  The above inequalities translate to the following conclusion:
  Whenever the following event happens,
  \eqn{
    \left\{\left | \frac{1}{n}\sum_{i=1}^n g(\<x_i,\beta_1 \>) 
        - \E{g(\< x,\beta_1\>)}  \right | > \e \right\},
  }
  at least one of the following events happens
  \eqn{
    \left\{ \frac{1}{n}\sum_{i=1}^n u_{\beta_2}(x_i) 
      - \E{u_{\beta_2}(x)} > \e/2 \right\}
    \ \ \text{or}\ \ 
    \left\{ \frac{1}{n}\sum_{i=1}^n l_{\beta_2}(x_i) 
      - \E{l_{\beta_2}(x)}  < -\e/2 \right\}.
  }

  Therefore, using the union bound on the above events,
  we may obtain
  \eq{ \label{eq::prob-bound}
    &\P{\sup_{\beta \in B(\tbeta)}\left | \frac{1}{n}\sum_{i=1}^n g(\<x_i,\beta \>) 
        - \E{g(\< x,\beta\>)}  \right | > \e } \\
    & \leq
    \P{ \max_{\beta\in\N_{\Delta_*}}\frac{1}{n}\sum_{i=1}^n u_{\beta}(x_i) 
      - \E{u_{\beta}(x)} > \e/2 }
    \nonumber \\
    &\ \ +\nonumber
    \P{\max_{\beta\in\N_{\Delta_*}} \frac{1}{n}\sum_{i=1}^n l_{\beta}(x_i) 
      - \E{l_{\beta}(x)}  <- \e/2 }
    .
  } 

  Note that the right hand side of the above 
  inequality has two terms 
  both of which are of the same form.
  For simplicity, we bound only the first one. 
  The bound for the second one follows from the 
  exact same steps.

  The relation between sub-Gaussian and sub-exponential norms
  \cite{vershynin2010introduction} allows us to write
  \eq{\label{eq::subGsubE}
    \| \|x \|_2 \|^2_\sg \leq  
    \| \|x \|^2_2 \|_\se 
    \leq& \sum_{i=1}^p \|x_i^2\|_\se, \\ \nonumber
    \leq &2\sum_{i=1}^p \|x_i\|^2_\sg
    \leq 2\kappa^2 p,
  }
  where the second step follows from the triangle inequality.
  Hence, we conclude that 
  $\| x \|_2 - \E{\| x \|_2}$ is a centered sub-Gaussian random variable 
  with norm upper bounded by $3\kappa\sqrt{p}$. 

  For $\e < 4/3$, we notice that the random variable 
  $u_\beta(x)=g(\< x,\beta\>) + \e {\| x\|_2}/{4\mu}$ 
  is also sub-Gaussian with norm
  \eqn{
    \|u_\beta(x)\|_\sg &\leq b +  \frac{\epsilon}{4\tmu} 3\kappa \\
    & \leq b +  \kappa/\tmu,
  }
  and consequently, 
  the centered random variable $u_\beta(x)- \E{u_\beta(x)}$ has
  the sub-Gaussian norm upper bounded by $2b+2\kappa/\tmu$.

  Then, by the Hoeffding-type inequality
  for the sub-Gaussian random variables, we obtain
  \eqn{
    \P{ 
      \frac{1}{n}\sum_{i=1}^n u_\beta(x_i)
      -\E{u_\beta(x)}
      > \e/2}  \leq &
    \expp{-c n\frac{\e^2}{(b+\kappa/\tmu)^2}}
  }
  for an absolute constant $c>0$. 

  By the same argument above, one can obtain the same result
  for the function $l_\beta(x)$. 
  Using Hoeffding bounds in \eqref{eq::prob-bound} along with the union bound
  over the net, we immediately obtain
  \eqn{
    \P{\sup_{\beta \in B(\tbeta)}\left | \frac{1}{n}\sum_{i=1}^n g(\<x_i,\beta \>) 
        - \E{g(\< x,\beta\>)}  \right | > \e } 
    \leq
    2\abs{\N_{\Delta_*}}
    \expp{-c n\frac{\e^2}{(b+\kappa/\tmu)^2}}
  }
  for some absolute constant $c$.

  Using a standard covering argument over the net $\N_{\Delta_*}$ 
  as given in Lemma \ref{lem::sphere},
  we have
  \eqn{
    \abs{\N_{\Delta_*}} \leq 
    \left(\frac{\delta\sqrt{p}}{\Delta_*}\right)^p = 
    \left(\frac{4\delta k \tmu p}{\e }\right)^p.
  }
  Combining this with the previous bound,
  and choosing
  \eqn{
    \e^2  = \frac{p}{n}\frac{(b+\kappa/\tmu)^2}{2c}
    \log\left(\frac{32c\delta^2k^2\tmu^2 pn}{(b+\kappa/\tmu)^2} \right)
  }
  we get 
  \eqn{
    &2\left(\frac{4\delta k \tmu p}{\e }\right)^p
    \expp{-c n\frac{\e^2}{(b+\kappa/\tmu)^2}}\\
    & = 2 \expp{-\frac{p}{2} \log \log
      \left(\frac{32c\delta^2k^2\tmu^2pn}{(b+\kappa/\tmu)^2}\right)}\\
    &\leq 2\expp{-p},
  }
  whenever $np > 51\max\{\chi,\chi^{-1} \}$ for $\chi = (b+\kappa/\tmu)^2/(c\delta^2k^2\tmu^2)$.

\end{proof}

\begin{lemma}[\cite{erdogdu2015convergence}]\label{lem::sphere}
  Let $B \subset \reals^p$ be the ball of radius $\delta$ 
  centered around some $\beta \in \reals^p$ 
  and $\N_\e$ be an $\e$-net over $B$. Then,
  \eqn{
    \abs{\N_\epsilon}
    \leq \left(\frac{\delta\sqrt{p}}{\e}\right)^p.
  }
\end{lemma}
\begin{proof}[Proof of Lemma \ref{lem::sphere}]
  The set $B$ can be contained in a $p$-dimensional cube of size $2\delta$. 
  Consider a grid over this cube with mesh width $2\e/\sqrt{p}$. 
  Then $B$ can be covered with at most $(2\delta/(2\e/\sqrt{p}))^p$ 
  many cubes of edge length $2\e/\sqrt{p}$. 
  If ones takes the projection of the centers of such cubes onto 
  $B$ and considers the circumscribed balls of radius $\e$, 
  we may conclude that $B$ can be covered with at most
  $$\left(\frac{2\delta}{2\e/\sqrt{p}}\right)^p$$
  many balls of radius $\e$.
\end{proof}
\begin{lemma}[Corollary 5.50 of \cite{vershynin2010introduction}]
  \label{lem::versh-sg}
  Let $\tx_1, \tx_2, ..., \tx_n$ be
  isotropic random vectors with sub-Gaussian norm
  upper bounded by $\kappa$.
  Then for every $t>0$, with probability at least $1-2\expp{-c_1t^2}$,
  the empirical covariance $\SigT$ satisfies,
  \eqn{
    \norm{\SigT - \I} \leq \max\{ \delta,\delta^2\}\ \ \ 
    \text{where}\ \ \ \delta = c_2\sqrt{\frac{p}{n}} + \frac{t}{\sqrt{n}}
  }
  where $c_1,c_2$ are constants depending only on $\kappa$.
\end{lemma}
\begin{remark}
  For $t = \sqrt{p/c_1}$, we get with probability at least $1-2\expp{-p}$,
  \eqn{
    \norm{\SigT - \I} \leq 
    C\sqrt{\frac{p}{n}}
  }
  where $$C=\left\{c_2 + \frac{1}{\sqrt{c_1}} \right\},$$
  and $n > C^2 p$. Here, $C$ only depends on $\kappa$.
\end{remark}
\begin{lemma}[Corollary 5.52 of \cite{vershynin2010introduction}]
  \label{lem::versh-heavy}
  Let $x_1, x_2, ..., x_n$ be random vectors with mean 0 and covariance $\Sig$ supported 
  on a centered Euclidean ball of radius $\sqrt{R}$,
  i.e., $\|x_i\|_2 \leq \sqrt{R}$. For $\e \in (0,1)$
  and $c>0$ an absolute constant,
  with probability at least $1-1/p^2$,
  the empirical covariance matrix satisfies
  \eqn{
    \norm{\SigH - \Sig} \leq \e \norm{\Sig},
  }
  for $n > cR\log(p)/(\e^2 \norm{\Sig}) $.
\end{lemma}

\end{document}